\documentclass[11pt]{article}

\usepackage[margin=1.25in]{geometry}

\usepackage[numbers]{natbib}
\usepackage{multicol}
\usepackage[bookmarks=true]{hyperref}
\usepackage{amsmath}
\usepackage{amssymb}
\usepackage{color}
\usepackage{graphicx}
\usepackage{appendix}
\graphicspath{{images/}}
\usepackage{amsmath}

\usepackage{xspace}

\usepackage{geometry}

\usepackage{setspace}

\usepackage{amsmath}
\usepackage{amsfonts}
\usepackage{amssymb}
\usepackage{amsthm}
\usepackage{bm}
\usepackage{bbm}
\usepackage{mathtools}
\usepackage{enumitem}
\usepackage{thmtools,thm-restate}
\usepackage{algorithm}
\usepackage{algorithmic}
\usepackage{color}
\usepackage{graphicx}
\usepackage{comment}

\def\Rbb{\mathbb{R}}

\def\R{\Rbb}

\def\*{\star}

\newcommand{\tr}[1]{ \mathrm{tr}\left( #1\right)}

\usepackage[dvipsnames]{xcolor}

\renewcommand{\tr}{{T}}

\newcommand{\paral}{{\!/\mkern-5mu/\!}}

\newcommand{\Lag}{\mathcal{L}}
\newcommand{\Ham}{\mathcal{H}}

\newcommand{\ma}{\mathbf{a}}
\newcommand{\p}{\mathbf{p}}
\newcommand{\q}{\mathbf{q}}
\newcommand{\qd}{{\dot{\q}}}
\newcommand{\qdd}{{\ddot{\q}}}
\newcommand{\uu}{\mathbf{u}}

\newcommand{\vv}{\mathbf{v}}

\newcommand{\x}{\mathbf{x}}
\newcommand{\xd}{{\dot{\x}}}
\newcommand{\xdd}{{\ddot{\x}}}
\newcommand{\y}{\mathbf{y}}

\newcommand{\z}{\mathbf{z}}

\newcommand{\f}{\mathbf{f}}
\newcommand{\h}{\mathbf{h}}
\newcommand{\g}{\mathbf{g}}

\newcommand{\zero}{\mathbf{0}}

\newcommand{\J}{\mathbf{J}}
\newcommand{\Jd}{{\dot{\J}}}

\newcommand{\A}{\mathbf{A}}
\newcommand{\B}{\mathbf{B}}

\newcommand{\D}{\mathbf{D}}

\newcommand{\G}{\mathbf{G}}
\newcommand{\mH}{\mathbf{H}}
\newcommand{\I}{\mathbf{I}}

\newcommand{\M}{\mathbf{M}}

\newcommand{\mP}{\mathbf{P}}
\newcommand{\mR}{\mathbf{R}}

\newcommand{\V}{\mathbf{V}}

\newcommand{\wt}[1]{{\widetilde{#1}}}

\usepackage{amsmath}  
\usepackage{amssymb}
\usepackage{accents}  
\usepackage{amsthm}

\theoremstyle{plain}
\newtheorem{theorem}{Theorem}[section]
\newtheorem{lemma}[theorem]{Lemma}
\newtheorem{proposition}[theorem]{Proposition}
\newtheorem{corollary}[theorem]{Corollary}

\theoremstyle{definition}
\newtheorem{definition}[theorem]{Definition}

\theoremstyle{remark}
\newtheorem{remark}[theorem]{Remark}

\makeatletter
\let\save@mathaccent\mathaccent
\newcommand*\if@single[3]{%
  \setbox0\hbox{${\mathaccent"0362{#1}}^H$}%
  \setbox2\hbox{${\mathaccent"0362{\kern0pt#1}}^H$}%
  \ifdim\ht0=\ht2 #3\else #2\fi
  }
\newcommand*\rel@kern[1]{\kern#1\dimexpr\macc@kerna}
\newcommand*\widebar[1]{\@ifnextchar^{{\wide@bar{#1}{0}}}{\wide@bar{#1}{1}}}
\newcommand*\wide@bar[2]{\if@single{#1}{\wide@bar@{#1}{#2}{1}}{\wide@bar@{#1}{#2}{2}}}
\newcommand*\wide@bar@[3]{%
  \begingroup
  \def\mathaccent##1##2{%
    \let\mathaccent\save@mathaccent
    \if#32 \let\macc@nucleus\first@char \fi
    \setbox\z@\hbox{$\macc@style{\macc@nucleus}_{}$}%
    \setbox\tw@\hbox{$\macc@style{\macc@nucleus}{}_{}$}%
    \dimen@\wd\tw@
    \advance\dimen@-\wd\z@
    \divide\dimen@ 3
    \@tempdima\wd\tw@
    \advance\@tempdima-\scriptspace
    \divide\@tempdima 10
    \advance\dimen@-\@tempdima
    \ifdim\dimen@>\z@ \dimen@0pt\fi
    \rel@kern{0.6}\kern-\dimen@
    \if#31
      \overline{\rel@kern{-0.6}\kern\dimen@\macc@nucleus\rel@kern{0.4}\kern\dimen@}%
      \advance\dimen@0.4\dimexpr\macc@kerna
      \let\final@kern#2%
      \ifdim\dimen@<\z@ \let\final@kern1\fi
      \if\final@kern1 \kern-\dimen@\fi
    \else
      \overline{\rel@kern{-0.6}\kern\dimen@#1}%
    \fi
  }%
  \macc@depth\@ne
  \let\math@bgroup\@empty \let\math@egroup\macc@set@skewchar
  \mathsurround\z@ \frozen@everymath{\mathgroup\macc@group\relax}%
  \macc@set@skewchar\relax
  \let\mathaccentV\macc@nested@a
  \if#31
    \macc@nested@a\relax111{#1}%
  \else
    \def\gobble@till@marker##1\endmarker{}%
    \futurelet\first@char\gobble@till@marker#1\endmarker
    \ifcat\noexpand\first@char A\else
      \def\first@char{}%
    \fi
    \macc@nested@a\relax111{\first@char}%
  \fi
  \endgroup
}
\makeatother

\usepackage{authblk}

\begin{document}

\title{ 
{\bf Optimization Fabrics}
}
\date{}
\author[1]{Nathan D. Ratliff}
\author[1]{Karl Van Wyk}
\author[1,2]{Mandy Xie}
\author[1,3]{\\Anqi Li}
\author[1,2]{Muhammad Asif Rana}
\affil[1]{NVIDIA}
\affil[2]{Georgia Tech}
\affil[3]{University of Washington}

\maketitle 

\begin{abstract}
Second-order differential equations define smooth system behavior. In general, there is no guarantee that a system will behave well when forced by a potential function, but in some cases they do and may exhibit smooth optimization properties such as convergence to a local minimum of the potential. Such a property is desirable in system design since it is inherently linked to asymptotic stability. This paper presents a comprehensive theory of {\em optimization fabrics} which are second-order differential equations that encode nominal behaviors on a space and are guaranteed to optimize when forced away from those nominal trajectories by a potential function. Optimization fabrics, or fabrics for short, can encode commonalities among optimization problems that reflect the structure of the space itself, enabling smooth optimization processes to intelligently navigate each problem even when the potential function is simple and relatively na\"{i}ve. Importantly, optimization over a fabric is asymptotically stable, so optimization fabrics constitute a building block for provably stable system design.

The majority of this paper is dedicated to the development of a tool set for the design and use of a broad class of fabrics called {\em geometric fabrics}. Geometric fabrics encode behavior as general nonlinear geometries which are covariant second-order differential equations with a special homogeneity property that ensures their behavior is independent of the system's speed through the medium. A class of Finsler Lagrangian energies can be used to define how these nonlinear geometries combine with one another. Furthermore, these geometric fabrics are closed under the standard operations of pullback and combination on a transform tree. For behavior representation, this class of geometric fabrics constitutes a broad class of spectral semi-sprays (specs), also known as Riemannian Motion Policies (RMPs) in the context of robotic motion generation, that captures both the intuitive separation between acceleration policy and priority metric critical for modular design and are inherently stable. Geometric fabrics, therefore, capture many of the most alluring and intuitive properties of earlier work on RMPs while remaining safe and easy to use by less experienced behavioral designers. 
\end{abstract}

\section{Introduction}

Riemannian Motion Policies (RMPs) \cite{ratliff2018riemannian} are powerful tools for modularizing robotic behavior. They enable designers behavior in parts, as focused position and velocity dependent acceleration policies and priority metrics in any of a number of specifically designed task spaces relevant to the problem. How these parts combine is defined by the mathematics of the RMP algebra and built in to the RMPflow algorithm.
The broadest class of RMPs form a flexible toolset for design and their utility has been demonstrated in a number of real-world collaborative settings where fast reaction and adaptation is critical, but their power lies primarily in their intuition.  
In particular, RMPs lack formal stability guarantees and can be dangerous for inexperienced designers. This lack of theoretical transparency forms a barrier to entry for anyone unfamiliar with the design principles, parameter tuning strategies, and damping heuristics commonly used to create good stable performance. 

This paper builds a comprehensive theory of these systems to both capture the intuition that makes them effective in practice but also prescribe a set of rules to ensure their safety, stability, and convergence. We introduce and fully characterize a class of systems called {\em optimization fabrics}. These systems define nominal behaviors that capture commonalities among tasks such as obstacle and joint limit avoidance, redundancy resolution, attractor shaping, and more. And importantly, one can optimize across a fabric by forcing the underlying system with a task-specific potential function whose minima define the task's goal. The overall optimization behavior is then influenced by the fabric, but the potential acts to push the system away from that nominal behavior toward accomplishing the task's goal. A system is defined as an optimization fabric if it is guaranteed to converge to a local minimum of the potential when forced, and the theory of optimization fabrics centers around both characterizing what systems form fabrics, and the construction of tools to aid the design such systems. The exposition below builds toward the characterization of a concrete and flexible class of fabrics called {\em geometric fabrics}. These fabrics are built from differential equations that share a form of geometric consistency well-suited to intuitive design.

The theory focuses on adding stability guarantees while maintaining the predominant characteristics of RMPs that make them intuitive and effective in practice. The most important of these characteristics are: (1) the separation of system behavior into nominal behavior shared across multiple tasks and a task-specific component that pushes away from that nominal behavior to accomplish task goals; and (2) the separation of acceleration policy design from priority metric design. Earlier work on Geometric Dynamical Systems (GDS) \cite{cheng2018rmpflow} addressed stability, but resulted in a framework requiring policies to be designed as {\em force} policies. Such policies could be translated to {\em acceleration} policies only through their interaction with the priority metric, which meant that the priority metric and acceleration behavior were inherently linked. For instance, designing a single GDS on a task space is straightfoward, but when combining that GDS with another using priority metrics designers had to adjust the force policy to compensate for changes in acceleration induced by modifications to the priority metric. 
Moreover, each subtask policy had to be designed as a potential function which meant that each of these subtask policies could potentially conflict with one another. For instance, achieving one policy's goal (potential minimization) results in the escalation of another policy's goal. Ultimately, this can lead to cases where the primary tasks, such as reaching a target, is not fully solved. 

The framework of optimization fabrics presented here provides clarity around these issues. It enables entirely acceleration based design ensuring that metrics are used solely as priorities in acceleration policy combination (resulting in metric weighted averages of individual policy accelerations). And each optimization fabric is characteristically unbiased meaning that a single separate task potential can be fully optimized over the top by the fabric without being fundamentally biased away from the local minimum by the fabric itself. Designers can thereby encode all subtask behaviors, such as obstacle avoidance and redundancy resolution, directly into the fabric and design the task's goal into the separate potential function to ensure it is reached at convergence. 


\subsection{Related work: An encompassing framework for existing techniques}

Collaborative settings are dynamic and partially observable and, therefore, require fast adaptation and reaction to both changes in the surrounding environment and changes in state estimation. Modern collaborative systems are designed around these real-time requirements, relying on an underlying layer of reactive control to interpret and adapt the plans computed by the computationally intensive planning layer \cite{2017_rss_system} on the fly. That design has inspired a new line of research exploring the development of local behavior generation tools stepping beyond reactivity to address additionally many of the considerations and constraints traditionally handled by the planning level. These tools encapsulate behaviors like collision and joint limit avoidance, or even shaping of the end-effector's behavior en route to a target \cite{ratliff2018riemannian}, while remaining reactive. This paper focuses on this reactive setting, and in particular, on the design of second-order (rather than first-order) dynamical systems which capture the innate second-order nature of robotic systems \cite{ClassicalMechanicsTaylor05}.

A classical method in this domain is operational space control \cite{KhatibOperationalSpaceControl1987} which can be viewed as a precursor to the broader area of geometric control, wherein behavior is shaped by controlling a system to behave like a specially designed virtual classical mechanical system (equivalently geodesics on a Riemannian manifolds, hence the adjective ``geometric''.). \cite{cheng2020RmpflowJournalArxiv} gives an in-depth review of geometric control (and operational space control) and constructs a generalized class of Geometric Dynamical Systems (GDS) that encompasses it. This new class of GDSs adds new modeling capacity, enabling users to design velocity-dependent priority or weight metrics for combining policies, while maintaining standard stability guarantees. That manuscript also reviews related literature and overviews their relationships. In the present work, we develop a much broader encompassing theory to contextually ground these previous techniques, showing that they can all be characterized as specific types of {\em optimization fabrics} with varying properties and limitations.

Our system categorization can be broadly characterized by the properties of their associated priority (weight) metrics and their policy types. (The metric defines how many policies on a common space trade off with one another when combined.)
Specifically, operational space control is a type of geometric control, and both model systems using classical mechanics (these systems are sometimes known as simple mechanical systems) \cite{ClassicalMechanicsTaylor05}. Classical mechanical systems are characterized by metrics (priority matrices) with position-only dependence. More broadly, Geometric Dynamical Systems (GDS) \cite{cheng2018rmpflow}, defined within the framework of Riemannian Motion Policies (RMPs) \cite{ratliff2018riemannian}, includes classical systems as a subclass, and introduces velocity dependence to their metrics. The present work introduces Lagrangian systems wherein the equations of motion are given by applications of the Euler-Lagrange equation to a stationary Lagrangian function (see Section~\ref{sec:ConservativeFabrics}). These Lagrangian systems are very similar to GDSs, but their equations are inherently covariant \cite{ratliff2020SpectralSemiSprays,ratliff2020FinslerGeometry} and do not require the notion of {\em structured} GDS as was done in the earlier \cite{cheng2018rmpflow} work. Both GDS and Lagrangian systems in their broadest generality do not offer much guidance on how to design systems. We will see here that a branch of mathematics called Finsler geometry (a generalization of Riemannian geometry), gives us a concrete yet flexible toolset for design within this more general Lagrangian context. These tools yield Finsler systems, which too have velocity-dependent metrics and thereby generalize classical mechanical systems. 

All of the above mentioned systems define behavior directly using specific types of dynamical systems as models. In each of these models, stability is provable because known conserved energy quantities constitute natural system Lyapunov functions \cite{khalil2002nonlinear}. Here, we expand on that idea by characterizing a broader class of systems that remain conservative but need not follow the natural dynamics of the energy function. These conservative systems are related to the above energy systems, indeed their priority metrics derive from the energy, but the corresponding policies can differ dramatically from the energy system's policy. We analyze the properties of such systems and characterize tools to aid their design, including an {\em energization transform} which takes an (almost) arbitrary differential equation and transforms it (with minimal modification) into a conservative system. These constructions ultimately lead to the development of what we call {\em geometric fabrics}, which are effectively general nonlinear geometry generators (described in Section~\ref{sec:NonlinearGeometries} energized using Finsler energies. Geometric fabrics constitute a concrete toolset for design within this broader encompassing class of conservative systems.

Each of these systems, when their equations satisfy specific properties we define here, can be used to design optimization fabrics which are provably convergent. To summarize, early literature focused on classical mechanical systems within the context of operational space control and geometric control. Later, \cite{cheng2018rmpflow} introduced variants of these systems called GDSs which added velocity dependence to their metrics, the general character of which, as described here, can be alternatively modeled as a Lagrangian system and more concretely using Finsler systems. And beyond those systems, we show in this work that much larger classes of geometric, energized, and most generally conservative systems exists adn are well-suited for system design. This framework offers a variety of new tools for behavior design, and the theoretical properties we study here provide intuition and rules for how to use them most effectively in practice.

Another class of second-order systems common in robotics is Dynamic Movement Primitives (DMPs) \cite{ijspeert2013DMPs}. DMPs, however, inherit their stability guarantees by reducing over time to linear systems as the nonlinear components vanish by design. In this paper, we focus instead on the framework extending from operational space control paradigms due to its compatibility with modern planning and optimization frameworks \cite{2017_rss_system,ratliff2018riemannian}.

\subsection{Classical mechanics: A simple example of an optimization fabric}

To gain some intuition around optimization fabrics, we first review the behavior of a classical mechanical system, one of the simplest forms of optimization fabric.

A second-order differential equation can often be considered an optimizing system. For instance, the damped classical mechanical equations of motion under a forcing potential function is guaranteed to optimize that potential function. Intuitively, the undamped system will maintain constant total energy (kinetic plus potential), and the damper will incrementally bleed energy from the system, reducing the total energy until it comes to rest at the bottom of the potential (local minimum of the potential and zero kinetic energy). Importantly, in addition to optimizing, the forced system's behavior en route to minimizing the potential is governed by the kinetic energy term (tantamount to the mass distribution). That kinetic energy defines a nominal behavior for the system that the system generally follows. However, when forced, that kinetic system becomes a medium across which the potential function is optimized. Change the potential and the system will still optimize the new potential. Change the kinetic energy, and the new system will again optimize the original potential but with a different transit behavior. 

This classical mechanical system is an example of what we call an optimization fabric, or fabric for short. A fabric more generally defines a nominal behavior that a forcing potential has to push away from. A fabric is unbiased in the sense that the resulting system will come to rest at a local minimum of the potential. 
For instance, the unforced classical mechanical system will remain at rest if it starts at rest; there is no biasing force that will explicitly push it from that state. On the other hand, a classical system experiencing a constant force in any particular direction independent of velocity would be biased, and that force bias would prevent it from optimizing a potential (the constant force would bias the system away from the potential's local minimum).

\subsection{Overview of the paper}

This paper develops a comprehensive theory of what we call {\em optimization fabrics}, nonlinear second-order differential equations defining a nominal task-independent system behavior that are {\em unbiased} in the sense that systems forced by a potential function are {\em guaranteed} to optimize that potential. We start with a review of the basic properties of {\em spectral semi-sprays} in Section~\ref{sec:SpecsAndTransformTrees} and then characterize optimization fabrics in their most general form in Section~\ref{sec:UnbiasedGeneralFabrics}. Those initial result are very broad and are not prescriptive of how to design fabrics. The remaining sections become increasingly specific in their characterization of subclasses of fabrics, building toward a description of what we call {\em geometric fabrics} in Section~\ref{sec:GeometricFabrics} which constitute a concrete and complete toolset for behavior design. 

Along the way, we encounter a class of {\em conservative fabrics} in Section~\ref{sec:ConservativeFabrics} which generalize the intuition behind classical mechanical systems given as an example above. These fabrics leverage the theoretical properties of Lagrangian systems and include Riemannian geometric systems and their Finsler geometric generalization. However, we show additionally that a much broader class of such systems modified with {\em zero work} terms can express conservative behavior while following almost arbitrary paths. More concretely we show that second-order differential equations can be {\em energized} to ensure that they conserve a given energy quantity. In general, energization slightly changes the behavior of the system, but we give conditions under which the basic nature of the underlying system remains unchanged by the energization process. Specifically, when the underlying system is a geometry, a broad class of equations which are in a sense path consistent (see Section~\ref{sec:NonlinearGeometries}), the resulting energized system expresses the same geometry of paths. When the energy used for energization is a Finsler energy (a generalization of kinetic energy that enables the modeling of directionally dependent metric tensors), then energization results in a {\em geometric fabric} that can be used for optimization.

With this theory, we can use geometric fabrics on a transform tree by populating nodes of the tree with fabric components and using the basic spectral semi-spray operations to combine them at the root. Many of the classes of fabric presented in this paper are known to be closed under these operations (see Section~\ref{sec:ClosureTheorem}) making transform trees a convenient tool for design.

In the context of Riemannian Motion Policies (RMPs), geometric fabrics form a flexible and provably stable class of RMPs that enable entirely acceleration-based design (see Section~\ref{sec:DesignWithGeometricFabrics} and specifically Section~\ref{subsec:acceleration_potential}). Acceleration-based design enables practitioners to design acceleration policies (geometry design) separately from their associated priorities specification (energy design), and these fabric components are combined as metric weighted averages. In Section~\ref{sec:speed_control} we show that because the underlying geometries are path consistent we can modulate the speed of the system to maintain a given measure of speed (such as end-effector or configuration space speed in a robotic manipulator) during execution. Section~\ref{sec:Experiments} gives experimental results verifying the theoretical properties outlined in this paper using both a 2D point mass system and a 3-dof planar manipulator. We show empirically that the acceleration-based design enables practitioners to construct complex fabrics in layers, adding increasingly complex components bit-by-bit without modifying the parameters of earlier components. With each added layer the underlying behavior becomes more sophisticated but always remains a fabric suitable for optimization. In this way, we design by hand a fabric for controlling a manipulator that is able to resolve redundancy naturally and shape the end-effector's behavior to lift from the surface of a floor before moving to a region above a target and descending from above. 


\subsection{A note on the presentation of preliminaries}

We assume many of the mathematical concepts presented, such as general nonlinear and Finsler geometries, are unfamiliar to most reading this papers. Many of the advanced mathematical presentations of related ideas require a background in differential geometry which has such extensive terminology and notation that it is inaccessible to most readers. We, therefore, in a separate papers \cite{ratliff2020SpectralSemiSprays,ratliff2020FinslerGeometry} present full derivations of the underlying results in the notation of advanced calculus to make them more accessible. That paper also introduces some nonstandard terminology that we view as more intuitive to make the concepts easier to digest. We make note of these modifications in Section~\ref{sec:terminology} on terminology, and we review these basic results when needed as we build toward the formulation and derivation of optimization fabrics.

\section{Preliminaries} \label{sec:preliminaries}

\subsection{Manifolds and notation}

The concepts and notation around differential geometry is complex and becomes a roadblock for most readers. In this paper, to make these ideas more accessible and clear, we choose a notation that avoids the typical coordinate-free or tensor based notations of differential geometry in lieu of an advanced calculus notation more common in engineering texts. Similarly, we assume the manifolds on which we derive our equations are well-defined with respect to the standard constructions of differential geometry, and derive our equations simply with respect to a chose coordinate system as is common in physics. By constructing concepts using concrete coordinate-independent quantities such as length measures, for which (nonlinear) change-of-variable transformations are intuitive, we can attain covariant coordinate-independent equations defined across the manifold in any coordinate system, analogous to physical equations in engineering. Appendix~\ref{apx:Manifolds} details this more concrete treatment of manifolds and Appendix~\ref{apx:CalculusNotation} defines concretely the advanced calculus notation we used in this paper within the context of a chosen nonlinear coordinate system.

We restrict attention here to smooth manifolds \cite{LeeSmoothManifolds2012} and assume unless otherwise stated that vector fields, functions, trajectories, etc. defined on the manifolds are smooth. Of importance is the notion of a smooth manifold with a boundary. The following definitions review the basic concepts and notation.

\begin{definition}
Let $\mathcal{X}$ be a manifold with a boundary (see \cite{LeeSmoothManifolds2012} for a definition\footnote{Intuitively, it can be thought of as a set theoretic boundary point in coordinates.}). Its boundary is denoted $\partial\mathcal{X}$ and we denote its {\em interior} as $\mathrm{int}(\mathcal{X}) = \mathcal{X}\backslash\partial\mathcal{X}$. 

\end{definition}
In general, when using the notation $\mathcal{X}$ we assume it is a manifold with a boundary $\partial\mathcal{X}$. However, unless otherwise explicitly stated, this notation includes the case where the boundary is empty $\partial\mathcal{X} = \emptyset$.


\begin{definition}
Let $\mathcal{X}$ denote a smooth manifold of dimension $n$. The space of all velocities at a point $\x$ is denoted $\mathcal{T}_\x\mathcal{X}$ and is known in manifold theory as the {\em tangent space} at $\x$. It is often convenient to reference the set of all available positions and velocities across a manifold. That space is known as the {\em tangent bundle} and is denoted $\mathcal{T}\mathcal{X} = \coprod_{\x\in\mathcal{X}}\mathcal{T}_\x\mathcal{X}$ where $\coprod$ denotes the disjoint union. I.e. $(\x,\xd)\in \mathcal{T}\mathcal{X}$ if and only if $\xd\in\mathcal{T}_\x\mathcal{X}$ for some $\x\in\mathcal{X}$. 
The boundary of a manifold $\partial\mathcal{X}$ is a separate smooth manifold of dimension $n-1$, with its own lower dimensional tangent bundle $\mathcal{T}\partial\mathcal{X} = \coprod_{\x\in\partial\mathcal{X}}\mathcal{T}_\x\partial\mathcal{X}$. Likewise, $\mathrm{int}(\mathcal{X})$ is the manifold of dimension $n$ consisting of all interior points with tangent bundle of consistent dimension $n$ denoted $\mathcal{T}\mathrm{int}(\mathcal{X}) = \coprod_{\x\in\mathrm{int}(\mathcal{X})}\mathcal{T}_\x\mathcal{X}$. With these definitions, the complete manifold with a boundary is understood to be the disjoint union of the separate interior and boundary manifolds $\mathcal{X}=\mathrm{int}(\mathcal{X})\coprod\partial\mathcal{X}$, and its tangent bundle is the disjoint union of the separate tangent bundles $\mathcal{T}\mathcal{X} = \mathcal{T}\mathrm{int}(\mathcal{X})\coprod\mathcal{T}\partial\mathcal{X}$.
%
\end{definition}
\begin{remark}
Statements below pertaining to all $(\x,\xd)\in\mathcal{T}\mathcal{X}$ are understood as all pertaining specifically to the two separate cases $(\x,\xd)\in\mathcal{T}\mathrm{int}(\mathcal{X})$ {\em and} $(\x,\xd)\in\mathcal{T}\partial\mathcal{X}$. That second case explicitly restricts the velocities to lie only tangent to the boundary's surface.
\end{remark}

\subsection{A taxonomy of terminology and results}
\label{sec:terminology}

There is a lot of terminology introduced in this paper, so this section collects the definitions into one place for quick reference. We also provide a list of the different types of fabrics encountered throughout the paper along with a concise taxonomy of their closure status under operations of the spec algebra defined in Section~\ref{sec:SpecAlgebra}. The section is for reference only and can be skipped on first read. 

\subsubsection{Terminology}

The following is a (partial) list of terms defined in this paper, each listed with a brief contextual note:
\begin{enumerate}
\item {\em Spectral semi-spray}, or {\em spec} for short: A pair $\big(\M(\x,\xd),\f(\x,\xd)\big)$ representing a differential equation $\M \xdd + \f = \zero$.
Operations of pullback and combination define an associated {\em spec algebra} over transform trees.
\item {\em Equations of motion}: The equation $\partial^2_{\xd\xd}\Lag + \partial_{\xd\x}\Lag\:\xd - \partial_\x\Lag = \zero$ given by applying the Euler-Lagrange equation to a stationary Lagrangian $\Lag(\x,\xd)$.
\item {\em Forcing a system}: Adding the gradient of a potential function, often along with a damper, to a spec. If the original spec is $\M\xdd + \f = \zero$, the forced system is $\M\xdd + \f = - \partial_\x \psi - \B\xd$.
\item {\em Optimization fabric}, or {\em fabric} for short: A special class of spec which is guaranteed to optimize the potential function when forced.
\item {\em Unbiased spec}: Fabrics all must be what we call {\em unbiased}. At a high-level this essentially means that they do not innately force a system when it comes to rest. Unbiased specs can, therefore, come to rest at the local minimum of a potential function and not be innately pushed away from it by the underlying fabric.
\item A {\em nonlinear geometry}: A geometrically consistent system of speed independent paths defined by a differential equation $\xdd + \h_2(\x, \xd) = \zero$, where $\h_2$ is homogeneous of degree 2 in velocity, known as the \textit{(geometry) generator}. Each generator has an associated \textit{geometric form} $P_\xd^\perp[  \xdd + \h_2(\x, \xd) ] = \zero$ known as the corresponding {\em geometric equation}.
\item {\em Finsler structure}: A Lagrangian that is positive when $\xd \neq \zero$ and homogeneous of degree 1 in velocity. Such Lagrangians define action integrals that are ``path length'' like, in that they are a positive measure that is invariant to time-reparameterization of the trajectory so that all trajectories following the same path give them same action measure. A Finsler structure can also be thought of as a {\em geometric} Lagrangian since it is a specific form of Lagrangian that captures the notion of a positive speed-independent path measure with a locally unique minimum. We use the term Finsler structure since it is common in the mathematical literature.
\item A \textit{Finsler geometry} is a nonlinear geometry defined by the equations of motion of a Finsler structure. Its generator is known as a {\em Finsler generator} and is given by the equations of motion of the corresponding Finsler energy.
\item {\em Lagrangian energy}: The Hamiltonian of a general Lagrangian. This energy is defined as $\Ham_\Lag = \partial_\xd\Lag^\tr\xd - \Lag$ and in general might differ from the Lagrangian. The energy is not to be confused with the Lagrangian itself. Only when the Hamiltonian matches the Lagrangian is that the case (such as in Finsler energies below). When the context of the Lagrangian is clear, we often use just {\em energy}.
\item {\em Finsler energy}: The energy form $L_e = \frac{1}{2} L_g^2$ of a Finsler structure $\Lag_g$. In this special case, the Lagrangian energy (Hamiltonian) associated with $\Lag_e$ is $\Lag_e$, itself. When the context of the Finsler structure $\Lag_g$ is known, we often refer to the Finsler energy as the {\em energy} or {\em energy form} of $\Lag_g$. The Finsler energy is always homogeneous of degree 2 in $\xd$ and can be used to define the Finsler geometry as $\Lag_g = \sqrt{2\Lag_e}$ if the defining properties hold for the derived Finsler structure $\Lag_g$.
\item {\em Bending a fabric}: Adding an energy conserving geometric term (homogeneous of degree 2) to a geometric fabrics's generator. The resulting generator generates a distinct geometry, but the resulting generator still conserves the same energy and remains a fabric. ``Bending'' is different from ``forcing'' (see above). 
\item {\em Metric tensor of a Lagrangian}: Defined as the Hessian of the Lagrangian $\Lag$ used to define the energy $\M = \partial_{\xd\xd}^2\Lag$.
\item {\em Energization} or {\em energization transform}: Transforming a differential equation by accelerating along the direction of motion to ensure that it conserves a given energy function. Energization can be viewed as an operation on a differential equation or, when relevant, the transformation of a differential equation into a fabric suitable for optimization.
\item {\em Energized fabric}: The energy conserving fabric resulting from an {\em energization transform}. When the energy is a Finsler energy, the fabric is a bent Finsler fabric referred to as a \textit{bent Finsler representation}. 
\end{enumerate}

\subsubsection{Types of fabrics}

The following is a list of the classes of fabrics defined throughout this paper:
\begin{enumerate}
\item \textit{Optimization fabric} or \textit{fabric} for short: A spec $(\M,\f)$ which optimizes when forced. i.e. $\M \xdd + \f + \partial_\x \psi + \B \xd = \zero$ optimizes psi.
\item \textit{Conservative fabric}: A fabric defined by an energy conserving equation of the form $\M_e\xdd + \f_e + \f_f = \zero$ where $\big(\M_e, \f_e\big)$ come from the energy Lagrangian $\Lag_e(\x, \xd)$ and $\f_f$ is a zero work contribution.
\item \textit{Lagrangian and Finsler fabrics}: The class of conservative fabrics defined by the equations of motion of an energy Lagrangian is known as a {\em Lagrangian fabric}. If the fabric is defined more specifically by a Finsler energy, it is called a {\em Finsler fabric}.
\item \textit{(Finsler) Energized fabrics}: Fabrics formed by energizing differential equations are known as {\em energized fabrics}. If a Finsler energy is used to energize the differential equation, it is called a {\em Finsler energized fabric}.
\item \textit{Geometric fabrics}: A fabric formed by energizing a geometry generator with a Finsler energy.
\end{enumerate}

\subsubsection{Closure under the spec algebra}

We say that a class of fabrics is {\em closed under tree operations}, or just {\em closed} if the context of operations is clear, if applying the operations to a fabric in that class results in another fabric of the same class (see Section~\ref{sec:SpecsOnTransformTrees}). With regard to the operations the spec algebra, we prove in Section~\ref{sec:ClosureTheorem} that Lagrangian fabrics, Finsler fabrics, and geometric fabrics are all closed. We also conjecture that the broader classes of general optimization fabrics, conservative fabrics, and Finsler energized fabrics are closed, although we do not prove it here. 

\section{Specs and transform trees} \label{sec:SpecsAndTransformTrees}

Let $\x$ and $\xd$ denote a position and velocity in a task space $\mathcal{X}$, assumed to be represented in some chosen coordinates. Optimization fabrics instances of systems of the form $\M\xdd + \f = \zero$, where $\M(\x,\xd)$ is symmetric and invertible and $\f(\x, \xd)$ are both functions of both position and velocity, with the property that they optimize when forced with a potential, which we discuss in detail below. We begin, though, with a brief overview of the broader class of differential equations.

\subsection{Spectral semi-sprays (specs)}

We call the broader class of such systems \textit{spectral semi-sprays}, or \textit{specs} for short, and denote them $(\M, \f)_\mathcal{X}$. When the space $\mathcal{X}$ is clear from context, we often drop the subscript. While specs are not the primary focus of this paper, since fabrics \textit{are} specs, we give a brief overview of the relevant ideas and manipulations here.

\subsubsection{Pullback and combination of specs}
\label{sec:SpecAlgebra}

A (non-spectral) semi-spray by itself is a differential equation of the form $\xdd + \h(\x,\xd) = \zero$ \cite{ratliff2020FinslerGeometry}, and we can easily write the above equation as $\xdd + \M^{-1}\f = \zero$ to fit this form, but it is critical explicitly represent $\M$ to track how it transforms between spaces, as we discuss next.

Given a differentiable map $\phi:\mathcal{Q}\rightarrow\mathcal{X}$ with action denoted $\x = \phi(\q)$, we can derive an explicit expression for the covariant transformation of a spec $(\M, \f)_\mathcal{X}$ on the codomain $\mathcal{X}$ to a spec $(\widetilde{\M}, \widetilde{\f})_\mathcal{Q}$ on the domain $\mathcal{Q}$. Denoting the map's Jacobian as $\J = \partial_\q\phi$, and noting $\xdd = \J\qdd + \Jd\qd$, the covariant transformation of left hand side of the differential equation $\M\xdd + \f = \zero$ represented by the spec is
\begin{align}
    &\J^\tr\big(\M\xdd + \f\big) = \J^\tr\Big(\M\big(\J\qdd + \Jd \qd\big) + \f\Big)\\
    &\ \ \ = \big(\J^\tr \M \J\big) \qdd + \J^\tr\big(\f + \Jd\qd\big) \\
    &\ \ \ = \widetilde{\M}\qdd + \widetilde{\f}.
\end{align}
where $\widetilde{\M} = \J^\tr \M \J$ and $\widetilde{\f} = \J^\tr\big(\f + \Jd\qd\big)$. This means we can define the following covariant pullback operation:
\begin{align}
    \mathrm{pull}_\phi (\M,\f)_\mathcal{X} 
    = \left(\J^\tr\M\J,\ \J^\tr\big(\f + \Jd\qd\big)\right)_\mathcal{Q}.
\end{align}
Similarly, since $\big(\M_1\xdd + \f_1\big) + \big(\M_2\xdd + \f_2\big) = \big(\M_1 + \M_2\big) + \big(\f_1 + \f_2\big)$, we can define an associative and commutative summation operation of the form
\begin{align}
    (\M_1,\f_1)_\mathcal{X} + (\M_2,\f_2)_\mathcal{X}
    = \big(\M_1 + \M_2,\ \f_1 + \f_2\big)_\mathcal{X}.
\end{align}
Note that the above operations are on what we call the natural form of specs. We can also view them in canonical form, which amounts to $(\M,-\M^{-1}\f)_\mathcal{X}^\mathcal{C}$ in the current setting where $\M$ is fully invertible. The expression $-\M^{-1}\f$ defines the acceleration, since the differential equation represented by the spec can be solved to give $\xdd = -\M^{-1}\f$. In terms of this canonical acceleration form, the summation operation computes the combined acceleration as a metric weighted average of individual accelerations:
\begin{align}
    (\M_1,\ma_1)^\mathcal{C}_\mathcal{X} + (\M_2,\ma_2)^\mathcal{C}_\mathcal{X}
    = \Big(\M_1 + \M_2,\ 
        (\M_1 + \M_2)^{-1}\big(\M_1\ma_1 + \M_2\ma_2\big)
      \Big)^\mathcal{C}_\mathcal{X}.
\end{align}

These operations of pullback and summation define the \textit{spec algebra}. 

\subsubsection{Specs on transform trees}
\label{sec:SpecsOnTransformTrees}

Differentiable maps can be composed together to make a tree of spaces rooted at the configuration space $\mathcal{C}$ known as a transform tree, with directed edges denoting the differentiable maps and nodes given by the spaces that result for those differentiable maps. See \cite{ratliff2020SpectralSemiSprays} for a detailed construction and analysis of transform trees.

It can be shown that, in the context of a transform tree, the space of specs becomes a compatible linear structure across the tree under the above defined spec algebra. That means the tree can be used to represent a composite spec at the root by placing specs on the tree's nodes and pulling them back and combining them recursively until a single resultant spec resides at the root \cite{ratliff2020SpectralSemiSprays}. 

Since specs form a compatible linear structure on the tree, this computation is independent of computational path, which means, at least for purposes of the theoretical analysis here, we can always think of the transform tree as star-shaped, under which each node has a single independent map linking directly from root to the node. (If it did not have this form, since the path to a given node is unique, we can perform what is known as a star-shaped transform wherein we define an equivalent tree under which each of those unique paths are represented explicitly as depth 1 branches containing a map defined by the composition of maps encountered along the path.) Such a star-shaped tree can be represented by a collection of $n$ differentiable maps $\phi_i:\mathcal{Q}\rightarrow\mathcal{X}_i$, $i=1,\ldots,n$. The pulled back and combined spec defined on $\mathcal{Q}$ representing a collection of specs $\{(\M_i,\f_i)\}_{i=1}^n$ defined on the $n$ spaces is
\begin{align}
    &\sum_{i=1}^n \mathrm{pull}_{\phi_i} (\M_i,\f_i)_{\mathcal{X}_i}
    = \left(\sum_i\J_i^\tr\M_i\J_i,\ \sum_i\J_i^\tr\big(\f_i + \Jd_i\qd\big)\right)_\mathcal{Q}
\end{align}
which can be viewed as a metric weighted average of individual pulled back specs in canonical form.

\begin{definition}[Closure under tree operations]
A class of specs is said to be {\em closed under tree operations}, or just {\em closed} for short when the set of operations is clear, when applying the operations to elements of the class results in an element of the same class.
\end{definition}
In particular, if a given class of specs is closed under tree operations, if the transform tree is populated with specs from that class, the spec that results at the root from pullback and combination is of the same class.

\section{Optimization fabrics} \label{sec:OptimizationFabrics}

A spec can be \textit{forced} by a position dependent potential function $\psi(\x)$ using
\begin{align}
    \M\xdd + \f = -\partial_\x\psi,
\end{align}
where the gradient $-\partial_\x\psi$ defines the force added to the system. 
In most cases, forcing an arbitrary spec does not result in a system that's guaranteed to converge to a local minimum of $\psi$. But when it does, we say that the spec is \textit{optimizing} and forms an \textit{optimization fabric} or \textit{fabric} for short. This section characterizes the class of specs that form fabrics using definitions and results of increasing specificity. These results are used in Subsection~\ref{sec:GeometricFabrics} to define what we call {\em geometric fabrics} which constitute a concrete set of tools for fabric design.

Note that the accelerations of a forced system $\xdd = -\M^{-1}\f - \M^{-1}\partial_\x\psi$ decompose into nominal accelerations of the system $-\M^{-1}\f$ and forced accelerations $-\M^{-1}\partial_\x\psi$.
Importantly, the spectrum of $\M$, therefore, plays a key role in defining how the potential force $-\partial_\x\psi$ acts to push the system away from the nominal path. It might be easy for potentials to push in some directions, but difficult to push in others. The metric $\M$ dictates the profile of {\em how} the potential function can push away from the system's nominal paths.

\subsection{Unbiased specs and general optimization fabrics} \label{sec:UnbiasedGeneralFabrics}

\begin{definition}
Let $\mathcal{X}$ be a manifold with boundary $\partial\mathcal{X}$ (possibly empty). A spec $(\M,\f)_\mathcal{X}$ is said to be {\em interior} if for any interior starting point 
$(\x_0, \xd_0) \in \mathcal{T}\mathrm{int}(\mathcal{X})$ the integral curve $\x(t)$ 
is everywhere interior $\x(t)\in\mathrm{int}(\mathcal{X})$ for all $t\geq0$.
\end{definition}

\begin{definition}
Let $\mathcal{X}$ be a manifold with a (possibly empty) boundary. A spec $\mathcal{S} = \big(\M,\f\big)_\mathcal{X}$ is said to be {\em rough} if all its integral curves $\x(t)$ converge: $\lim_{t\rightarrow\infty}\x(t) = \x_\infty$ with $\x_\infty\in \mathcal{X}$ (including the possibility $\x_\infty\in\partial\mathcal{X}$). If $\mathcal{S}$ is not rough, but each of its {\em damped variants} $\mathcal{S}_\B = \big(\M,\f+\B\xd\big)$ is, where $\B(\x,\xd)$ is smooth and positive definite, the spec is said to be {\em frictionless}. A frictionless spec's damped variants are also known as {\em rough variants} of the spec.
\end{definition}


\begin{definition}
Let $\psi(\x)$ be a smooth potential function with gradient $\partial_\x\psi$ and let $(\M,\f)_\mathcal{X}$ be a spec. Then $(\M,\f+\partial_\x\psi)$ is the spec's {\em forced variant} and we say that we are {\em forcing} the spec with potential $\psi$. We say that $\psi$ is {\em finite} if $\|\partial_\x\psi\|<\infty$ everywhere on $\mathcal{X}$.
\end{definition}


\begin{definition}
A spec $\mathcal{S}$ forms a {\em rough fabric} if it is rough when forced by a finite potential $\psi(\x)$ and each convergent point $\x_\infty$ is a Karush–Kuhn–Tucker (KKT) solution of the constrained optimization problem $\min_{\x\in\mathcal{X}}\psi(\x)$. A forced spec is a {\em frictionless fabric} its rough variants form rough fabrics.
\end{definition}
\begin{lemma}
If a spec $\mathcal{S}$ forms a rough (or frictionless) fabric then (unforced) it is a rough (or frictionless) spec.
\end{lemma}
\begin{proof}
The zero function $\psi(\x) = 0$ is a finite potential under which all points in $\mathcal{X}$ are KKT solutions. If $\mathcal{S}$ is a rough fabric, then all integral curves of $\mathcal{S}$ are convergent and it must be a rough spec. 
\end{proof}



\begin{definition}\label{def:BoundaryConformingSpec}
A spec $\mathcal{S} = \big(\M,\f\big)_\mathcal{X}$ is {\em boundary conforming} if the following hold: 
\begin{enumerate}
\item $\mathcal{S}$ is interior.
\item $\M(\x,\vv)$ and $\f(\x,\vv)$ are finite for all\footnote{In this definition, we emphasize in the notation that the manifold contains both its interior and the boundary. In practice, proofs often treat both cases separately.} $(\x,\vv)\in\mathcal{T}\mathcal{X} = \mathcal{T}\mathrm{int}(\mathcal{X}) \cup \mathcal{T}\partial\mathcal{X}$.
\item For every tangent bundle trajectory $(\x(t), \vv(t)) \in \mathcal{T}\mathcal{X}$ with $\x\rightarrow\x_\infty\in\partial\mathcal{X}$, we have
\begin{align}
\lim_{t\rightarrow\infty}\left\|\M^{-1}(\x, \vv)\right\|<\infty
\ \ \mbox{and}\ \  
\lim_{t\rightarrow\infty}\big\|\V_\infty^\tr\f(\x,\vv)\big\| < \infty  
\end{align}
where $\V_\infty$ is a matrix for which $\mathrm{span}(\V_\infty) = \mathcal{T}_{\x_\infty}\mathcal{X}$.
\end{enumerate}
A {\em boundary conforming metric} is a metric satisfying conditions (2) and (3) of this definition. Additionally, $\f$ is said to be {\em boundary conforming with respect to $\M$} if $\M$ is a boundary conforming metric and $\big(\M, \f\big)_\mathcal{X}$ forms a boundary conforming spec. 
\end{definition}
Note that this definition of boundary conformance implies that $\M$ either approaches a finite matrix along trajectories limiting to the boundary or it approaches a matrix that is finite along Eigen-directions parallel with the boundary's tangent space but explodes to infinity along the direction orthogonal to the tangent space. This means that either $\M^{-1}_\infty$ is full rank or it is reduced rank and its column space spans the boundary's tangent space $\mathcal{T}_{\x_\infty}\partial\mathcal{X}$.



\begin{definition} \label{def:Unbiased}
A boundary conforming spec $\big(\M,\f\big)_\mathcal{X}$ is {\em unbiased} if for every convergent trajectory $\x(t)$ with $\x\rightarrow\x_\infty$ we have
$\V_\infty^\tr\f(\x,\xd) \rightarrow \zero$ where $\V_\infty$ is a matrix for which $\mathrm{span}(\V_\infty) = \mathcal{T}_{\x_\infty}\mathcal{X}$.
If the spec is not unbiased, we say it is {\em biased}. We often refer to the term $\f$ alone as either biased or unbiased when the context of $\M$ is clear. Note that all unbiased specs must also be boundary conforming by definition. We, therefore, often specify only that the spec is unbiased with the implicit understanding that it is also, by definition, boundary conforming.
\end{definition}
\begin{remark}
In the above definition, when $\x_\infty\in\mathrm{int}(\mathcal{X})$, the basis $\mathcal{V}_\infty$ contains a full set of $n$ linearly independent vectors, so the condition $\V_\infty^\tr\f(\x, \xd)\rightarrow\zero$ implies $\f(\x, \xd)\rightarrow\zero$. This is not the case for $\x_\infty\in\partial\mathcal{X}$. 
\end{remark}
\begin{remark}
The matrix $\V_\infty$ naturally defines two linearly independent subspaces, the column space, and the left null space. It is often convenient to decompose a vector into two components $\f = \f^\paral + \f^\perp$, one lying in the column space $\f^\paral$ and the other lying in the left null space $\f^\perp$. A number of the statements below are phrased in terms of such a decomposition.
\end{remark}

Since the definition of unbiased is predicated on the spec being boundary conforming, the spectrum of the metric $\M$ is always finite in the relevant directions (all directions for interior points and directions parallel to the boundary for boundary points). The property of being unbiased is, therefore, linked to zero acceleration within the relevant subspaces. This property is used in the following theorem to characterize general fabrics.


\begin{theorem}[General fabrics] \label{thm:GeneralFabrics}
Suppose $\mathcal{S} = \big(\M,\f\big)_\mathcal{X}$ is a 
boundary conforming spec. Then $\mathcal{S}$ forms a rough 
fabric if and only if it is unbiased and 
converges when forced by a potential $\psi(\x)$ with $\|\partial\psi\|<\infty$ on $\x\in\mathcal{X}$. 
\end{theorem}
\begin{proof}
The forced spec defines the equation
\begin{align}\label{eqn:ForcedGeneralSystem}
    \M\xdd + \f = -\partial_\x\psi.
\end{align}

We will first assume $\f$ is unbiased. Since $\x$ converges, we must have $\xd\rightarrow\zero$ which means $\xdd\rightarrow\zero$ as well. 

If $\x(t)$ converges to an interior point $\x_\infty\in\mathrm{int}(\mathcal{X})$, $\M$ is finite so $\M\xdd\rightarrow\zero$ since $\xdd\rightarrow\zero$. And since the spec is unbiased we have $\f(\x,\xd)\rightarrow\zero$ as $\xd\rightarrow\zero$. Therefore, the left hand side of Equation~\ref{eqn:ForcedGeneralSystem} approaches $\zero$, so $\partial_\x\psi\rightarrow\zero$ satisfying (unconstrained) KKT conditions. 

Alternatively, if $\x(t)$ converges to a boundary point $\x_\infty\in\partial\mathcal{X}$,
then we can analyze the expression
\begin{align} \label{eqn:BoundaryLimit}
    \xdd = -\M^{-1}\big(\f + \partial_\x\psi\big)\rightarrow\zero.
\end{align}
since $\xdd\rightarrow\zero$.
Since $\M$ is boundary conforming, the inverse metric limit $\M_\infty^{-1}$ is finite, and since $\partial\psi$ is also finite on $\mathcal{X}$, the term $\M^{-1}\partial\psi$ converges to the finite vector $\M^{-1}_\infty\partial\psi_\infty$. Therefore, by Equation~\ref{eqn:BoundaryLimit} we also have $\M^{-1}\f \rightarrow \M^{-1}_\infty\f_\infty = -\M^{-1}_\infty\partial\psi_\infty$. At the limit, $\M^{-1}_\infty$ has full rank across $\mathcal{T}_{\x_\infty}\partial\mathcal{X}$, so the above limit equality implies $\f_\infty^\paral = -\partial\psi_\infty^\paral$, were $\f_\infty^\paral$ and $\partial\psi_\infty^\paral$ are the components of $\f_\infty$ and $\partial\psi_\infty$, respectively, lying in the boundary's tangent space $\mathcal{T}_{\x_\infty}\partial\mathcal{X}$. Since $\f$ is unbiased and $\x_\infty\in\partial\mathcal{X}$, the boundary parallel component $\f_\infty^\paral = \zero$ so it must be that $\partial\psi_\infty^\paral = \zero$ as well. Therefore, $\partial\psi_\infty$ must either be orthogonal to $\mathcal{T}_{\x_\infty}\partial\mathcal{X}$ or zero. If it is zero, the KKT conditions are automatically satisfied. If it is nonzero, since $\x(t)$ is interior, $-\f(\x,\xd)$ must be interior near the boundary, so $-\partial\psi_\infty$ must be exterior facing and balancing $-\f_\infty$. That orientation in addition to the orthogonality implies that the limiting point satisfies the (constrained) KKT conditions.

Finally, to prove the converse, assume $\f$ is biased. Then there exists a point $\x^*\in\mathcal{X}$ for which $\f(\x^*,\zero)\neq \zero$. We can easily construct an objective potential with a unique global minimum at $\x^*$. The forced system, in this case, cannot come to rest at $\x^*$ since $\f$ is nonzero there, so $(\M,\f)$ is not guaranteed to optimize and is not a fabric.
\end{proof}

The above theorem characterizes the most general class of fabric and shows that all fabrics are necessarily unbiased in the sense of Definition~\ref{def:Unbiased}. The theorem relies on hypothesizing that the system always converges when forced, and is therefore more of a template for proving a given system forms a fabric rather than a direct characterization. Proving convergence is in general nontrivial. The specific fabrics we introduce below will prove convergence using energy conservation and boundedness properties.

Note that the above theorem does not place restrictions on whether or not in the limit the metric is finite in Eigen-directions orthogonal to the boundary surface's tangent space. In practice, it can be convenient to allow metrics to raise to infinity in those directions, so the effects of forces orthogonal to the boundary's surface are increasingly reduced by the increasingly large mass. Such metrics can induce smoother optimization behavior when optimizing toward local minima on a boundary surface.

\subsection{Conservative fabrics} \label{sec:ConservativeFabrics}

Conservative fabrics, frictionless fabrics which conserve a well-defined form of energy, constitute a large class of practical fabrics. Here we define a class of energies through the Lagrangian formalism and extend the notion of boundary conformance to these energy Lagrangians. We give some essential lemmas on energy conservation and present the basic result on conservative fabrics in Proposition~\ref{prop:ConservativeFabrics}. Unlike the most general theorem on fabrics given above in Theorem~\ref{thm:GeneralFabrics}, here we no longer need to hypothesize that the system converges. The conservative properties enable us to prove convergence by effectively using the energy as a Lyapunav function \cite{khalil2002nonlinear}.  

Two standard classes of conservative fabrics that naturally arise are Lagrangian and Finsler fabrics (the latter a subclass of the former), both of which are defined by the equations of motion of their respective classes of energy Lagrangians. The energy-conservation properties of Lagrangian systems are well-understood from classical mechanics \cite{ClassicalMechanicsTaylor05}; this observation simply links those result to our fabric framework. In subsection~\ref{sec:EnergizationAndEnergizedFabrics}, though, we show how to create a much broader class of conservative system by applying an {\em energization} transform to a differential equation resulting in a new type of conservative fabric known as an energized fabric.

\begin{definition}
A stationary Lagrangian $\Lag(\x,\xd)$ is {\em boundary conforming (on $\mathcal{X}$)} if its induced equations of motion $\partial^2_{\xd\xd}\Lag\,\xdd + \partial_{\xd\x}\Lag\,\xd - \partial_\x\Lag = \zero$ under the Euler-Lagrange equation form a boundary conforming spec $\mathcal{S}_\Lag = \big(\M_\Lag, \f_\Lag\big)_\mathcal{X}$ where $\M_\Lag = \partial^2_{\xd\xd}\Lag$ and $\f_\Lag = \partial_{\xd\x}\Lag\,\xd - \partial_\x\Lag$. This spec is known as the {\em Lagrangian spec} associated with $\Lag$. $\Lag$ is additionally {\em unbiased} if $\mathcal{S}_\Lag$ is unbiased. Since unbiased specs are boundary conforming by definition, an unbiased Lagrangian $\Lag$ is implicitly boundary conforming as well.
\end{definition}

\begin{definition}
Let $\Lag_e(\x, \xd)$ be a stationary Lagrangian with Hamiltonian $\Ham_e(\x, \xd) = \partial_\xd\Lag_e^\tr\,\xd - \Lag_e$. $\Lag_e$ is an {\em energy Lagrangian} if $\M_e = \partial^2_{\xd\xd}\Lag_e$ is full rank, $\Ham_e$ is nontrivial (not everywhere zero), and $\Ham_e(\x, \xd)$ is finite on $\mathrm{int}(\mathcal{X})$.
An energy Lagrangian's equations of motion $\M_e\xdd + \f_e = \zero$ are often referred to as its {\em energy equations} with spec denoted $\mathcal{S}_e = \big(\M_e, \f_e\big)_\mathcal{X}$.
\end{definition}

\begin{definition}[Boundary intersecting trajectory]
Let $\x(t)$ be a trajectory. If there exists a $t_0 < \infty$ such $\x(t_0)\in\partial\mathcal{X}$ and $\xd\notin\mathcal{T}_{\x(t_0)}\partial\mathcal{X}$ it is said to be {\em boundary intersecting} with {\em intersection time} $t_0$.
\end{definition}

\begin{definition}[Energy boundary limiting condition]
Let $\Lag_e$ be an energy Lagrangian with energy $\Ham_e$. If for every boundary intersecting trajectory $\x(t)$ with intersection time $t_0$ we have $\lim_{t\rightarrow t_0} \Ham_e(\x, \xd) = \infty$, we say that $\Ham_e$ satisfies the {\em boundary limiting condition}.
\end{definition}

\begin{remark}
When designing boundary conforming energy Lagrangians one must ensure the Lagrangian's spec is interior. To attain an interior spec, it is often helpful to design energies that prevent energy conserving trajectories from intersecting $\partial\mathcal{X}$. It can be shown that an energy Lagrangian's spec is interior if and only if its energy is boundary limiting, i.e. the energy approaches infinity for any boundary intersecting trajectory.
\end{remark}

\begin{lemma}\label{lma:EnergyTimeDerivative}
Let $\Lag_e$ be an energy Lagrangian with energy $\Ham_e = \partial_\xd\Lag_e^\tr \xd - \Lag_e$. The energy time derivative is 
\begin{align} \label{eqn:EnergyTimeDerivative}
    \dot{\Ham}_e = \xd^\tr\big(\M_e\xdd + \f_e\big),
\end{align}
where $\M_e$ and $\f_e$ come from the Lagrangian's equations of motion $\M_e\xdd + \f_e = \zero$.
\end{lemma}
\begin{proof}
The calculation is a straightforward time derivative of the Hamiltonian:
\begin{align}
    \dot{\Ham}_{\Lag_e} 
    &= \frac{d}{dt}\big[\partial_\xd\Lag_e - \Lag_e\big] \\
    &= \big(\partial_{\xd\xd}\Lag_e\,\xdd + \partial_{\xd\x}\Lag_e\,\xd\big)^\tr \xd + \partial_\xd\Lag_e^\tr\,\xdd - \big(\partial_\xd\Lag_e\,\xdd + \partial_\x\Lag_e^\tr\xd \big) \\
    &= \xd^\tr\Big(\partial_{\xd\xd}\Lag_e\,\xdd + \partial_{\xd\x}\Lag_e\,\xd - \partial_\x\Lag_e^\tr\xd\Big) \\
    &= \xd^\tr\big(\M_e\xdd + \f_e\big).
\end{align}
\end{proof}

\begin{lemma} \label{lma:EnergyConservation}
    Let $\Lag_e$ be an energy Lagrangian. Then $\M_e\xdd + \f_e + \f_f = \zero$ is energy conserving if and only if $\xd^\tr\f_f = \zero$. We call such a term a {\em zero work modification}. Such a spec $\mathcal{S} = \big(\M_e, \f_e + \f_f\big)$ is said to be a {\em conservative} spec under energy Lagrangian $\Lag_e$.
\end{lemma}
\begin{proof}
This energy is conserved if its time derivative is zero. Substituting $\xdd = -\M_e^{-1}\big(\f_e + \f_f\big)$ into the Equation~\ref{eqn:EnergyTimeDerivative} of Lemma~\ref{lma:EnergyTimeDerivative} and setting it to zero gives
\begin{align}
    \dot{\Ham}_{\Lag_e} 
    &= \xd^\tr\Big(\M_e \big(-\M_e^{-1}(\f_e + \f_f)\big)+ \f_e\Big) \\
    &= \xd^\tr\big(-\f_e - \f_f + \f_e\big) \\
    &= \xd^\tr\f_f = 0.
\end{align}
Therefore, energy is conserved if and only if final constraint holds.
\end{proof}

\begin{proposition}[Conservative fabrics] \label{prop:ConservativeFabrics}
Suppose $\mathcal{S} = \big(\M_e, \f_e + \f_f\big)_\mathcal{X}$ is a conservative unbiased spec under energy Lagrangian $\Lag_e$ with zero work term $\f_f$. Then $\mathcal{S}$ forms a frictionless fabric.
\end{proposition}
\begin{proof}
Let $\psi(\x)$ be a lower bounded finite potential function. Using Lemma~\ref{lma:EnergyTimeDerivative} we can derive an expression for how the total energy $\Ham_e^{\psi} = \Ham_e + \psi(\x)$ varies over time:
\begin{align}
    \dot{H}_e^\psi 
    &= \dot{H}_e + \dot{\psi} \\
    &= \xd^\tr\big(\M_e\xdd + \f_e\big) + \partial_\x\psi^\tr \xd \\
    \label{eqn:TotalEnergyTimeDerivative}
    &= \xd^\tr\big(\M_e\xdd + \f_e + \partial_\x\psi\big).
\end{align}
With damping matrix $\B(\x, \xd)$ let $\M_e\xdd + \f_e + \f_f = -\partial_\x\psi - \B\xd$ be the forced and damped variant of the conservative spec's system. Plugging that system into Equation~\ref{eqn:TotalEnergyTimeDerivative} gives
\begin{align} \label{eqn:SystemEnergyDecrease}
    \nonumber
    \dot{H}_e^\psi
    &= \xd^\tr\Big(\M_e\big(-\M_e^{-1}(\f_e + \f_f + \partial_\x\psi + \B\xd)\big) + \f_e + \partial_\x\psi\Big)\\
    \nonumber
    &= \xd^\tr\Big(-\f_e - \partial_\x\psi - \B\xd + \f_e + \partial_\x\psi\Big) - \xd^\tr\f_f \\
    &= -\xd^\tr\B\xd
\end{align}
since all terms cancel except for the damping term. When $\B$ is strictly positive definite, the rate of change is strictly negative for $\xd\neq\zero$. Since $\Ham_e^{\psi} = \Ham_e + \psi$ is lower bounded and $\dot{\Ham}_e^{\psi}\leq\zero$, we must have $\dot{\Ham}_e^\psi = -\xd^\tr\B\xd \rightarrow \zero$ which implies $\xd\rightarrow\zero$. Therefore, the system is guaranteed to converge.  
Since the system is additionally boundary conforming and unbiased, by Theorem~\ref{thm:GeneralFabrics} it forms a fabric.

Finally, when $\B = \zero$ Equation~\ref{eqn:SystemEnergyDecrease} shows that total energy is conserved, so a system that starts with nonzero energy cannot converge. Therefore, the undamped system is a frictionless fabrics with rough variants defined by the added damping term.
\end{proof}

\begin{corollary}[Lagrangian and Finsler fabrics]
If $\Lag_e(\x, \xd)$ is an unbiased energy Lagrangian, then $\mathcal{S}_e = \big(\M_e, \f_e\big)_\mathcal{X}$ forms a frictionless fabric known as a {\em Lagrangian fabric}. When $\Lag_e$ is a Finsler energy (see Definition~\ref{def:FinslerStructure}), this fabric is known more specifically as a {\em Finsler fabric}.
\end{corollary}
\begin{proof}
$\M_e \xdd + \f_e = \zero$ is conservative by Lemma~\ref{lma:EnergyConservation} with $\f_f = \zero$. Since it is additionally unbiased, by Proposition~\ref{prop:ConservativeFabrics} it forms a frictionless fabric.
\end{proof}



\subsection{Energization and energized fabrics}
\label{sec:EnergizationAndEnergizedFabrics}

The following lemma collects some results around common matrices and operators that arise when analyzing energy conservation.
\begin{lemma}\label{lma:EnergyProjection}
Let $\Lag_e$ be an energy Lagrangian. Then with $\p_e = \M_e\xd$,
\begin{align}
    \mR_{\p_e} = \M_e^{-1} - \frac{\xd\,\xd^\tr}{\xd^\tr\M_e\xd}
\end{align}
has null space spanned by $\p_e$ and
\begin{align}
    \mR_\xd = \M_e - \frac{\p_e\p_e^\tr}{\p_e^\tr\M_e^{-1}\p_e}
\end{align}
has null space spanned by $\xd$. These matrices are related by
$\mR_\xd = \M_e\mR_{\p_e}\M_e$ and the matrix $\M_e^{-1}\mR_\xd = \M_e\mR_{\p_e} = \mP_e$ is a projection operator of the form
\begin{align} \label{eqn:EnergyProjector}
    \mP_e = \M_e^{\,\frac{1}{2}}\Big[\I - \hat{\vv} \hat{\vv}^\tr\Big]\M_e^{-\frac{1}{2}}
\end{align}
where $\vv = \M_e^{\,\frac{1}{2}}\xd$ and $\hat{\vv} = \frac{\vv}{\|\vv\|}$ is the normalized vector. Moreover, $\xd^\tr\mP_e\f = \zero$ for all $\f(\x,\xd)$.
\end{lemma}
\begin{proof}
Right multiplication of $\mR_{\p_e}$ by $\p_e$ gives:
\begin{align}
    \mR_{\p_e} \p_e 
    &= \left(\M_e^{-1} - \frac{\xd\,\xd^\tr}{\xd^\tr\M_e\xd}\right) \M_e\xd \\
    &= \xd - \xd \left(\frac{\xd^\tr\M_e\xd}{\xd^\tr\M_e\xd}\right) = \zero,
\end{align}
so $\p_e$ lies in the null space. Moreover, the null space is no larger since each matrix is formed by subtracting off a rank 1 term from a full rank matrix. 

The relation between $\mR_{\p_e}$ and $\mR_\xd$ can be shown algebraically
\begin{align}
    \M_e\mR_{\p_e}\M_e 
    &= \M_e\left(\M_e^{-1} - \frac{\xd\,\xd^\tr}{\xd^\tr\M_e\xd}\right)\M_e
    = \M_e - \frac{\M_e\xd\,\xd^\tr\M_e}{\xd^\tr\M_e\xd} \\
    &= \M_e - \frac{\p_e\p_e^\tr}{\p_e^\tr\M_e^{-1}\p_e} = \mR_\xd.
\end{align}
Since $\M_e$ has full rank, $\mR_\xd$ has the same rank as $\mR_{\p_e}$ and its null space must be spanned by $\xd$ since $\mR_\xd\xd = \M_e\mR_{\p_e}\M_e\xd = \M_e\mR_{\p_e}\p_e = \zero$.

With a slight algebraic manipulation, we get
\begin{align}
    \M_e\mR_{\p_e} 
    &= \M_e\left(\M_e^{-1} - \frac{\xd\,\xd^\tr}{\xd^\tr\M_e\xd}\right) \\
    &= \M_e^{\,\frac{1}{2}} \left(\I - \frac{\M_e^{\,\frac{1}{2}}\xd\,\xd^\tr\M_e^{\,\frac{1}{2}}}{\xd^\tr\M_e^{\,\frac{1}{2}}\M_e^{\,\frac{1}{2}}\xd}\right) \M_e^{-\frac{1}{2}} \\
    &= \M_e^{\,\frac{1}{2}} \left(\I - \frac{\vv\vv^\tr}{\vv^\tr\vv}\right) \M_e^{-\frac{1}{2}} \\
    &= \mP_e
\end{align}
since $\frac{\vv\,\vv^\tr}{\vv^\tr\vv} = \hat{\vv}\hat{\vv}^\tr$. Moreover, 
\begin{align}
    \mP_e\mP_e 
    &= 
    \M_e^{\,\frac{1}{2}} \left(\I - \frac{\vv\vv^\tr}{\vv^\tr\vv}\right) \M_e^{-\frac{1}{2}}
    \M_e^{\,\frac{1}{2}} \left(\I - \frac{\vv\vv^\tr}{\vv^\tr\vv}\right) \M_e^{-\frac{1}{2}} \\
    &= \M_e^{\,\frac{1}{2}} \left(\I - \frac{\vv\vv^\tr}{\vv^\tr\vv}\right) \left(\I - \frac{\vv\vv^\tr}{\vv^\tr\vv}\right) \M_e^{-\frac{1}{2}} \\
    &= \M_e^{\,\frac{1}{2}} \left(\I - \hat{\vv}\hat{\vv}^\tr\right) \M_e^{-\frac{1}{2}} = \mP_e,
\end{align}
since $\mP_\perp = \I - \hat{\vv}\hat{\vv}^\tr$ is an orthogonal projection operator. Therefore, $\mP_e^2 = \mP_e$ showing that it is a projection operator as well.

Finally, we also have 
\begin{align}
    \xd^\tr\mP_e 
    &= \xd^\tr\M_e\mR_{\p_e}
    = \xd^\tr\M_e\left(\M_e^{-1} - \frac{\xd\,\xd^\tr}{\xd^\tr\M_e\xd}\right)\\
    &= \left[\xd^\tr - \left(\frac{\xd^\tr\M_e\xd}{\xd^\tr\M_e\xd}\right)\xd^\tr\right]
    = \zero.
\end{align}
Therefore, for any $\f$, we have $\xd^\tr\mP_e\f = \zero$.
\end{proof}

\begin{corollary}
Let $\Lag_e$ be an energy Lagrangian and let $\wt{\f}_f(\x, \xd)$ by any forcing term. Then using the projected term $\f_f = \mP_e\wt{\f}_f$ the forced equations of motion
\begin{align} \label{eqn:ForcedEnergySystem}
    \M_e\xdd + \f_e + \f_f = \zero 
\end{align}
are energy conserving.
\end{corollary}
\begin{proof}
By Lemma~\ref{lma:EnergyProjection}, $\xd^\tr\f_f = \xd^\tr\mP_e\f = \zero$, so by Lemma~\ref{lma:EnergyConservation}, Equation~\ref{eqn:ForcedEnergySystem} is energy conserving.
\end{proof}

\begin{proposition}[System energization] \label{prop:SystemEnergization}
Let $\xdd + \h(\x, \xd) = \zero$ be a differential equation, and suppose $\Lag_e$ is any energy Lagrangian with equations of motion $\M_e\xdd + \f_e = \zero$ and energy $\Ham_e$. Then $\xdd + \h(\x, \xd) + \alpha_{\Ham_e}\xd = \zero$ is energy conserving when
\begin{align}\label{eqn:EnergizationTransformAlpha}
    \alpha_{\Ham_e} = -(\xd^\tr\M_e\xd)^{-1}\xd^\tr\big[\M_e\h - \f_e\big],
\end{align}
and differs from the original system by only an acceleration along the direction of motion. The new system can be expressed as:
\begin{align} \label{eqn:ZeroWorkEnergizationForm}
   \M_e\xdd + \f_e + \mP_e\big[\M_e\h - \f_e\big] = \zero.
\end{align}
This modified system is known as an {\em energized system}. Moreover, the operation of {\em energizing} the original system is known as an {\em energization transform} and is denoted using spec notation as
\begin{align}
    \mathcal{S}_{\h}^{\Lag_e} = \Big(\M_e, \f_e + \mP_e\big[\M_e\h - \f_e\big]\Big)_\mathcal{X} = \mathrm{energize}_{\Lag_e}\Big\{\mathcal{S}_\h\Big\},
\end{align}
where $\mathcal{S}_\h = \big(\I, \h\big)_\mathcal{X}$ is the spec representation of $\xdd + \h(\x,\xd) = \zero$.
\end{proposition}
\begin{proof}
Equation~\ref{eqn:EnergyTimeDerivative} of Lemma~\ref{lma:EnergyTimeDerivative} gives the time derivative of the energy. Substituting a system of the form $\xdd = -\h(\x,\xd) - \alpha_{\Ham_e}\xd$, setting it to zero, and solving for $\alpha_{\Ham_e}$ gives
\begin{align}
&\dot{\Ham}_e = \xd^\tr\Big[\M_e\big(-\h - \alpha_{\Ham_e}\xd\big) + \f\Big] = \zero \\
&\Rightarrow -\xd^\tr\M_e\h - \alpha_{\Ham_e} \xd^\tr\M_e\xd + \xd^\tr\f_e = \zero \\
&\Rightarrow \alpha_{\Ham_e} = -\frac{\xd^\tr\M_e\h - \xd^\tr\f_e}{\xd^\tr\M_e\xd}\\
&\ \ \ \ \ \ \ \ \ \ = 
-(\xd^\tr\M_e\xd)^{-1}\xd^\tr\big[\M_e\h - \f_e\big].
\end{align}
This result gives the formula for Equation~\ref{eqn:EnergizationTransformAlpha}. Substituting this solution for $\alpha_{\Ham_e}$ back in gives
\begin{align}
    \xdd 
    &= -\h + \left(\frac{\xd^\tr}{\xd^\tr\M_e\xd}\big[\M_e\h - \f_e\big]\right) \xd \\
    &= -\h + \left[\frac{\xd\,\xd^\tr}{\xd^\tr\M_e\xd}\right]\big(\M_e\h - \f_e\big).
\end{align}
Algebraically, it helps to introduce $\h = \M_e^{-1}\big[\f_f + \f_e\big]$ to first express the result as a difference away from $\f_e$; in the end we will convert back to $\h$. Doing so and moving all the terms to the left hand side of the equation gives
\begin{align}
    &\ \ \ \ \xdd + \h - \left[\frac{\xd\,\xd^\tr}{\xd^\tr\M_e\xd}\right]\big(\M_e\h - \f_e\big) = \zero \\
    &\Rightarrow \xdd + \M_e^{-1}\big[\f_f + \f_e\big] - \left[\frac{\xd\,\xd^\tr}{\xd^\tr\M_e\xd}\right]\Big(\M_e\M_e^{-1}\big[\f_f + \f_e\big] - \f_e\Big) = \zero \\
    &\Rightarrow \M_e \xdd + \f_f + \f_e - \M_e\left[\frac{\xd\,\xd^\tr}{\xd^\tr\M_e\xd}\right]\Big(\f_f + \big(\f_e - \f_e\big)\Big) = \zero \\
    &\Rightarrow \M_e \xdd + \f_e + \M_e\left[\M_e^{-1} - \frac{\xd\,\xd^\tr}{\xd^\tr\M_e\xd}\right] \f_f = \zero \\
    &\Rightarrow \M_e \xdd + \f_e + \M_e\mR_{\p_e}\big(\M_e\h - \f_e\big) = \zero,
\end{align}
where we substitute $\h = \M_e\f_f - \f_e$ back in. By Lemma~\ref{lma:EnergyProjection} $\M_e\mR_{\p_e} = \mP_e$, so we get Equation~\ref{eqn:ZeroWorkEnergizationForm}.
\end{proof}

\begin{lemma} \label{lma:LinearCombinationsUnbiased}
Suppose $\M_e$ is boundary conforming. Then if $\big(\M_e, \f\big)$ and $\big(\M_e, \g\big)$ are both unbiased, then $\big(\M_e, \alpha \f + \beta\g\big)$ is unbiased. 
\end{lemma}
\begin{proof}
Suppose $\x(t)$ is a convergent trajectory with $\x\rightarrow\x_\infty$ If $\x_\infty\in\mathrm{int}(\mathcal{X})$, then $\f\rightarrow \zero$ and $\g\rightarrow \zero$, so $\alpha \f + \beta\g\rightarrow \zero$. If $\x_\infty\in\partial\mathcal{X}$, then $\f = \f_1 + \f_2$ with $\f_1\rightarrow \zero$ and $\f_2\rightarrow \f_\perp\perp\mathcal{T}_{\x_\infty}\partial\mathcal{X}$, and similarly with $\g = \g_1 + \g_2$ with $\g_1\rightarrow \zero$ and $\g_2\rightarrow \g_\perp\perp\mathcal{T}_{\x_\infty}\partial\mathcal{X}$. Then $\alpha \f + \beta \g = \big(\alpha \f_1 + \beta \f_2\big) + \big(\alpha \g_1 + \beta \g_2\big) \rightarrow \alpha \g_1 + \beta \g_2$ which is orthogonal to $\mathcal{T}_{\x_\infty}\partial\mathcal{X}$ since each component in the linear combination is. Therefore, $\alpha \f + \beta \g$ is unbiased.
\end{proof}

For completeness, we state a standard result from linear algebra here (see \cite{gallier2020LinearAlgAndOpt} Proposition 8.7 for a proof). This result will be used in the below lemmas leading up to the proof that energized systems are unbiased.
\begin{proposition}\label{lem:MatrixVectorProductNormBound}
For every matrix $\A$, there exists a contant $c>0$ such that $\|\A\uu\|\leq c\|\uu\|$ for every vector $\uu$, where $\|
\cdot\|$ can be any norm.
\end{proposition}
In the below, for concreteness, we take the matrix norm to be the Frobenius norm.

\begin{lemma} \label{lma:ProjectedUnbiasedIsUnbiased}
Let $\M_e$ be boundary conforming. Then if $\f$ is unbiased with respect to $\M_e$, $\mP_e\f$ is unbiased, where 
$\mP_e = \M_e^{\,\frac{1}{2}}\big[\I - \hat{\vv}\,\hat{\vv}^\tr\big]\M_e^{-\frac{1}{2}}$ is the projection operator defined in Equation~\ref{eqn:EnergyProjector}.
\end{lemma}
\begin{proof}
Let $\x(t)$ be convergent with $\x\rightarrow\x_\infty$ as $t\rightarrow\infty$. If $\x_\infty\in\mathrm{int}(\mathcal{X})$, then since $\mP_e$ is finite on $\mathrm{int}(\mathcal{X})$, by Lemma~\ref{lem:MatrixVectorProductNormBound} there exists a constant $c>0$ such that $\|\mP_e\f\|\leq c\|\f\|$. Since $\|\f\|\rightarrow 0$ as $t\rightarrow 0$, it must be that $\|\mP_e\f\|\rightarrow 0$. 

Consider now the case where $\x_\infty\in\partial\mathcal{X}$. Since $\mP_e$ is a projection operator, for every $\z$ there must be a decomposition into linearly independent components $\z = \z_1 + \z_2$ with $\z_2$ in the kernel such that $\mP_e\z = \mP_e\z_1 + \mP_e\z_2 = \z_1$ (so that $\mP_e^2\z = \mP_e\z_1 \z_1 = \mP_e\z$). Thus, if $\z$ has a property if and only if all elements of a decomposition have that property, then $\mP_e\z = \z_1$ must have that property as well. Specifically, if $\z$ lies in a subspace (respectively, lies orthogonal to a subspace) then $\mP_e\z = \z_1$ lies in that subspace as well (respectively, lies orthogonal to a subspace).

The property of being unbiased is defined by the behavior of the decomposition $\f = \f^\paral + \f^\perp$, where those components are respectively parallel and perpendicular to $\mathcal{T}_{\x_\infty}\partial\mathcal{X}$ in the limit. Following the above outlined subscript convention to characterize the behavior of the projection, we have
\begin{align}
    \mP_e\f 
    &= \mP_e\Big(\f^\paral + \f^\perp\Big)
    = \mP_e\f^\paral + \mP_e\f^\perp\\
    &= \f^{\paral}_1 + \f^{\perp}_1
    \rightarrow \f^{\perp}_1
\end{align}
since $\f$ being unbiased implies $\f^\paral\rightarrow \zero$ and hence $\f^\paral_1\rightarrow \zero$.

Therefore, in both cases, $\mP_e\f$ satisfies the conditions of being unbiased if $\f$ does. 
\end{proof}

\begin{lemma} \label{lma:BoundaryAlignedMetricTransformUnbiased}
Suppose $\M_e(\x,\xd)$ is boundary conforming and boundary aligned. If $\big(\I,\h\big)$ is unbiased then $\big(\M_e, \M_e\h\big)$ is unbiased.
\end{lemma}
\begin{proof}
Let $\x(t)$ be convergent with $\x\rightarrow\x_\infty$ as $t\rightarrow\infty$. If $\x_\infty\in\mathrm{int}(\mathcal{X})$, then since $\M_e$ is finite on $\mathcal{T}\mathrm{int}(\mathcal{X})$, by Lemma~\ref{lem:MatrixVectorProductNormBound} there exists a constant $c>0$ such that $\|\M_e\h\|\leq c\|\h\|$. Since $\|\h\|\rightarrow 0$ as $t\rightarrow 0$, it must be that $\|\M_e\h\|\rightarrow 0$. 

Since $\M_e$ is boundary aligned, there is a subset of Eigenvectors that span the tangent space in the limit as $\x(t)\rightarrow\x_\infty\in\partial\mathcal{X}$. Let $\V_\paral$ denote a matrix containing the Eigenvectors that limit to spanning the tangent space, with $\D_\paral$ a diagonal matrix containing the corresponding Eigenvalues. Likewise, let $\V_\perp$ contain the remaining (perpendicular in the limit) Eigenvectors, with Eigenvalues $\D_\perp$. The metric decomposes as $\M_e = \M_e^\paral + \M_e^\perp$ with $\M_e^\paral = \V_\paral\D_\paral\V_\paral^\tr$ and $\M_e^\perp = \V_\perp\D_\perp\V_\perp^\tr$. Since $\h$ is unbiased, we can express $\h = \h^\paral + \h^\paral = \V_\paral\y_1 + \V_\perp\y_2$, where $\y_1$ and $\y_2$ are coefficients, and $\y_1\rightarrow \zero$ as $t\rightarrow\infty$. Therefore,
\begin{align}
    \M_e\h &= \Big(\M_e^\paral + \M_e^\perp\Big)\big(\h^\paral + \h^\perp\big)\\
    &= \Big(
        \V_\paral\D_\paral\V_\paral^\tr + \V_\perp\D_\perp\V_\perp^\tr
    \Big) \big(\V_\paral\y_1 + \V_\perp\y_2\big) \\
    &= \V_\paral\D_\paral \y_1 + \V_\perp\D_\perp \y_2 \\
    &\rightarrow \V_\perp\z 
    \ \ \ \perp\ \ \ \mathcal{T}_{\x_\infty}\partial\mathcal{X}
    \ \ \ \mbox{where $\z = \D_\perp \y_2$}.
\end{align}
Therefore, when $\x_\infty\in\partial\mathcal{X}$, the component parallel to the tangent space vanishes in the limit.
\end{proof}

\begin{definition}
A metric $\M(\x, \xd)$ is said to be {\em boundary aligned} if for any convergent $\x(t)$ with $\x_\infty\in\partial\mathcal{X}$ the limit $\lim_{t\rightarrow\infty}\M^{-1}(\x,\xd) = \M_\infty^{-1}$ exists and is finite, and $\mathcal{T}_{\x_\infty}\partial\mathcal{X}$ is spanned by a subset of Eigen-basis of $\M_\infty^{-1}$.
\end{definition}

\begin{lemma} \label{lma:EnergizationPreservesUnbiasedProperty}
Suppose $\mathcal{S}_\h = \big(\I, \h\big)$ is an unbiased (acceleration) spec and $\Lag_e$ is an unbiased energy Lagrangian with boundary aligned metric $\M_e$. Then $\mathcal{S}_\h^{\Lag_e} = \mathrm{energize}_{\Lag_e}\big\{\mathcal{S}_\h\big\}$ is unbiased.
\end{lemma}
\begin{proof}
$\M_e$ is boundary conforming by hypothesis on $\Lag_e$. Moreover, the energized equation takes the form
\begin{align}
    \M_e\xdd + \f_e + \mP_e\big[\M_e\h - \f_e\big] = \zero
\end{align}
as shown in Proposition~\ref{prop:SystemEnergization}. By Lemma~\ref{lma:BoundaryAlignedMetricTransformUnbiased} $\M_e\h$ is unbiased since $\h$ is unbiased, and likewise $\M_e\h - \f_e$ is unbiased by Lemma~\ref{lma:LinearCombinationsUnbiased} since $\f_e$ is unbiased by hypothesis on $\Lag_e$. That means $\mP_e\big[\M_e\h - \f_e\big]$ is unbiased as well by Lemma~\ref{lma:ProjectedUnbiasedIsUnbiased}. Finally, by again applying Lemma~\ref{lma:LinearCombinationsUnbiased} we see that the entirety of $\f_e + \mP_e\big[\M_e\h - \f_e\big]$ is unbiased.
\end{proof}

\begin{theorem}[Energized fabrics] \label{thm:EnergizedFabrics}
Let $\Lag_e$ be an unbiased energy Lagrangian with boundary aligned $\M_e = \partial^2_{\xd\xd}\Lag_e$ and lower bounded energy $\Ham_e$, and let $\big(\I, \h\big)$ be an unbiased spec. Then the energized spec $\mathcal{S}_\h^{\Lag_e} = \mathrm{energize}_{\Lag_e}\big\{\mathcal{S}_\h\big\}$ given by Proposition~\ref{prop:SystemEnergization} forms a frictionless fabric.
\end{theorem}
\begin{proof}
This result follows from 
Proposition~\ref{prop:ConservativeFabrics}, and Lemma~\ref{lma:EnergizationPreservesUnbiasedProperty}.

\end{proof}

\section{Generalized nonlinear geometries and geometric fabrics}

Above we showed that a broad class of conservative fabrics is formed by energizing differential equations. In general, though, the energization transform given by Proposition~\ref{prop:SystemEnergization} may change the behavior of the underlying system since systems generally are not invariant to the speed of traversal. Specifically, systems follows different paths when forced to speed up or slow down. An intuitive example of such a system is a particle in a gravitational field. When traveling at just the right speed around a mass, the particle can orbit. However, if it speeds up or slows down along the direction of motion it will either break orbit or spiral into the mass; in both cases, the path will necessarily change. 

In this section, we develop a special class of differential equations called {\em geometry generators} that maintain geometric path consistency despite the speed of traversal. For this special class of equation, the path behavior of a system is {\em invariant} under energization, so an energized fabric formed by energizing a geometry generator follows the {\em same} path as the original system, and is itself a geometry generator. We call this special class of energized fabrics {\em geometric fabrics}.

We start with a brief review of the construction and basic properties of general nonlinear geometries here for completeness. For more detail see \cite{ratliff2020FinslerGeometry}.

\subsection{Nonlinear geometries: Generators and geometric equations} \label{sec:NonlinearGeometries}

\begin{definition}[Geometry generator]
A {\em geometry generator}, or \textit{generator} for short, is an ordinary second-order differential equation of the form
\begin{align}
    \xdd + \h_2(\x,\xd) = \zero,
\end{align}
where $\h_2(\x,\xd)$ is a smooth, covariant, map $\h_2:\R^d\times\R^d\rightarrow\R^d$ that is {\em positively homogeneous of degree 2} in velocities in the sense $\h_2(\x, \alpha\xd) = \alpha^2\h_2(\x,\xd)$ for $\alpha>0$. Solutions to the generator are called {\em generating} solutions or trajectories, and if those trajectories are guaranteed to conserve an known energy quantity, they are known as {\em energy levels}.
\end{definition}
As an ordinary second-order differential equation, a generator's solutions are unique for specific initial values $(\x_0,\xd_0)$. A generator's degree 2 homogeneity means that all solutions to initial value problems of the form $(\x_0, \alpha \hat{\vv}_0)$, where $\hat{\vv}_0$ is any unit vector defining a direction in space, follow the same path. Said another way, all generating trajectories starting from a given point $\x_0$ with initial velocity pointing in the same direction $\xd_0 = \alpha\hat{\vv}_0$ follow the same path.

Generators are often called {\em sprays} in differential geometry \cite{shen2013differential}, although the term generator is more explicit about its role generating the geometry of a geometric equation, as we define next.

\begin{definition}[Geometric equation]
The {\em geometric equation} corresponding to a generator of the form $\xdd + \h_2(\x,\xd) = \zero$ is an equation of the form
\begin{align} \label{eqn:geometricEquation}
    \mP_\xd^{\perp}\big[\xdd + \h_2(\x,\xd)\big] = \zero,
\end{align}
where $\mP_\xd^{\perp}$ is an projector projecting orthogonally to $\xd$. Solutions to this geometric equation are known as {\em geometric} solutions or trajectories.
\end{definition}
\begin{remark}
Any matrix $\A_\xd$ with nullspace spanned by $\xd$ would suffice in this definition, but we choose the projector for clarity of its role.
\end{remark}

While the generator's solutions are unique but follow the same path when the initial conditions' velocities point in the same direction, the geometric equation is a redundant equation (redundancy coming from the reduced rank matrix $\mP_\xd^\perp$) whose solutions are the set of {\em all} trajectories following that single path. It can be shown that every solution to the generator equation is a smooth time-reparameterization of any generating trajectory. 

Moreover, for a given geometric trajectory there exists a generating trajectory whose instantaneous velocity matches the geometric trajectory's velocity  at a given point $\x$, so the geometric trajectory can be viewed as speeding up and slowing down along the direction of motion to smoothly move between generating solutions, hence the name generating solution. When these generating solutions form energy levels, we can say the geometric solution speeds up and slows down along the direction of motion to smoothly move between energy levels of the system.

The geometric equation fully characterizes a geometry of paths. The equivalence class of solutions to problems $(\x_0,\alpha\hat{\vv}_0)$ for $\alpha>0$ are (locally) the set of reparameterizations of a one-dimensional smooth submanifold of the space. We say that the collection of these speed-independent paths defines a nonlinear geometry on the space.

We point out without proof that all solutions of the nullspace geometric equation given in Equation~\ref{eqn:geometricEquation} can be expressed as
\begin{align} \label{eqn:explicitGeometricEquation} 
    \xdd = -\h_2(\x,\xd) + \gamma(t)\xd,
\end{align}
where $\gamma(t)$ is any smooth function of time. This is an explicit expression showing that geometric solutions are formed by speeding up and slowing down along the direction of motion $\xd$.

\subsection{Finsler geometry}

Finsler geometry is the study of nonlinear geometries whose geometric equation is defined by the equations of motion of {\em Finsler structure}.
\begin{definition}[Finsler structure] \label{def:FinslerStructure}
A {\em Finsler structure} is a stationary Lagrangian $\Lag_g(\x,\xd)$ with the following properties:
\begin{enumerate}
    \item Positivity: $\Lag_g(\x,\xd) > 0 $ for all $\xd \neq \zero$.
    \item Homogeneity: $\Lag_g$ is {\em positively homogeneous of degree 1} in velocities in the sense $\Lag_g(\x,\alpha\xd) = \alpha \Lag_g(\x,\xd)$ for $\alpha>0$.
    \item Energy tensor invertibility: $\partial^2_{\xd\xd}\Lag_e$ is everywhere invertible, where $\Lag_e = \frac{1}{2}\Lag_g^2$.
\end{enumerate}
$\Lag_e$ is known as the {\em energy form} of $\Lag_g$, and is sometimes called the {\em Finsler energy}. 
\end{definition}
\begin{remark}
Note that the first two conditions together mean that $\Lag(\x,\zero) = 0$.
\end{remark}

{\em Finsler structures} might be more descriptively termed {\em geometric Lagrangians} \cite{ratliff2020FinslerGeometry} since their geometric properties (invariance to time-reparameterization) stem directly from conditions on the Lagrangian and their effect on the resulting action. But these functions are commonly known as Finsler structures in the literature, so we maintain that tradition here. Note that many texts replace that third condition (invertibility of the energy tensor) with a positive definiteness requirement. That positive definiteness is difficult to satisfy in practice in Finsler structure design, and the basic proofs in our construction detailed in \cite{ratliff2020FinslerGeometry} require only invertibility, so we use the looser requirement in our definition.

\begin{proposition}[Energy of a Finsler geometry]
Let $\Lag_e = \frac{1}{2}\Lag_g^2$ be the Finsler energy of Finsler structure $\Lag_g$. The Hamiltonian (conserved quantity) of $\Lag_e$ is $\Lag_e$. Specifically,
\begin{align}
    \Ham_e = \partial_\xd\Lag_e^\tr\xd - \Lag_e = \Lag_e.
\end{align}
\end{proposition}
\begin{proof}
By Euler's theorem on homogeneous functions, if $f(\y)$ is homogeneous of degree $k$, then $\partial_\y f^\tr \y = k f(\y)$. That means for Lagrangians $\Lag(\x,\xd)$, we have
\begin{align}
    \partial_\xd\Lag^\tr\xd = k \Lag,
\end{align}
which means 
\begin{align}
    \Ham &= \partial_\xd\Lag^\tr\xd - \Lag = k \Lag - \Lag = (k-1) \Lag
\end{align}
Since $\Lag_g$ is homogeneous of degree 1 in $\xd$, $\Lag_e$ is homogeneous of degree 2 in $\xd$, so for $\Lag_e$ the above analysis means $\Ham_e = (2-1)\Lag_e = \Lag_e.$
\end{proof}

Equivalent forms of higher order energy can be defined as well via $\Lag_e^k = \frac{1}{k}\Lag_g^k$ and all of the results below also hold, but we treat only $k=2$ here for clarity of exposition.

We state the following fundamental results on Finsler geometry without proof and point the reader to \cite{ratliff2020FinslerGeometry} for details.

\begin{lemma}[Homogeneity of the Finsler energy tensor] \label{lma:FinslerEnergyHomogeneity}
Let $\Lag_g$ be a Finsler structure with energy form $\Lag_e = \frac{1}{2}\Lag_g^2$ and let $\M_e\xdd + \f_e = \zero$ be its equations of motion. Then $\M_e$ is homogeneous of degree 0 and $\f_e$ is homogeneous of degree 2.
\end{lemma}
The above lemma means that $\M_e$ is dependent on velocity $\xd$ only through its norm $\hat{\xd}$, i.e. rescaling $\xd$ does not affect the energy tensor.

\begin{theorem}[Finsler geometry generation] \label{thm:FinslerGeometries}
Let $\Lag_g$ be a Finsler structure with energy $\Lag_e = \frac{1}{2}\Lag_g^2$. The equations of motion of $\Lag_e$ define a geometry generator whose geometric equation is given by the equations of motion of $\Lag_g$.
\end{theorem}
These results mean that Finsler structures $\Lag_g$ define nonlinear geometries of paths whose generators define energy levels of the energy $\Lag_e$.





\subsection{Geometric fabrics} \label{sec:GeometricFabrics}

Here we discuss the properties of fabrics that arise from energizing geometry generators using the theory outlined in Section~\ref{sec:OptimizationFabrics}.

The expression in Equation~\ref{eqn:ZeroWorkEnergizationForm} shows that energization can be viewed as a zero work modification to the energy equations of motion. When the original differential equation $\xdd + \h = \zero$ is a geometry generator (i.e. $\h$ is homogeneous of degree 2) and $\Lag_e$ is Finsler, then the Finsler equations of motion, the original differential equation, and energized equation are all geometry generators, and importantly the energized equation in \ref{eqn:ZeroWorkEnergizationForm} generates a geometry equivalent to the original equation's generated geometry. We can, therefore, view this zero work modification as {\em bending} the Finsler geometry to match the desired geometry without affecting the system energy. This result is summarized in the following proposition.

\begin{corollary}[Bent Finsler Representation] \label{cor:BentFinslerRepresentation}
Suppose $\h_2(\x, \xd)$ is homogeneous of degree 2 so that $\xdd + \h_2(\x, \xd) = \zero$ is a geometry generator, and let $\Lag_e$ be a Finsler structure (and therefore also a Finsler energy). Then the energized system $\M_e\xdd + \f_e + \mP_e\big[\M_e\h_2 - \f_e\big] = \zero$ is a geometry generator whose geometry matches the original system's geometry. Since the Finsler system $\M_e\xdd + \f_e = \zero$ is a geometry generator as well, we can view the energized system as a {\em zero work geometric modification} to the Finsler geometry, what we call a {\em bending} of the geometric system.
\end{corollary}
\begin{proof}
The energized system takes the form $\xdd + \wt{\h}_2(\x, \xd) = \zero$ where
\begin{align}
    \wt{\h}_2 = \M_e^{-1}\f_e + \mR_e\big[\M_e\h_2 - \f_e\big]
\end{align}
since $\mP_e = \M_e\mR_{\p_e}$.
The energy $\Lag_e$ is Finsler, so $\f_e$ is homogeneous of degree 2 and $\M_e$ is homogeneous of degree 0, which means the first term in combination is homogeneous of degree 2. Moreover, $\mR_{\p_e} = \M_e^{-1} - \frac{\xd\,\xd^\tr}{\xd^
\tr\M_e\xd}$ is homogeneous of degree 0 since the numerator and denominator scalars would cancel in the second term when $\xd$ is scaled. Therefore, the energized system in its entirety forms a geometry generator.

Since $\xdd + \h_2(\x, \xd) = \zero$ is a geometry generator, instantaneous accelerations along the direction of motion $\xd$ do not change the paths taken by the system. So $\xdd + \wt{\h}_2(\x, \xd) = \zero$ forms a generator whose geometry matches the original geometry defined by $\xdd + \h_2(\x, \xd) = \zero$.
\end{proof}
Corollary~\ref{cor:BentFinslerRepresentation} shows that geometries are invariant under energization transforms performed with respect to Finsler energies.

Note that for the energized system to be a generator, we need two properties: first, the original system must be a generator, and second, the energy must be Finsler. If the energy is not Finsler, the resulting energized system will still follow the same paths as the original geometry (since by definition it is formed by accelerating along the direction of motion), but it will not be itself a generator (the resulting differential equation will not produce path aligned trajectories when solved for differing initial speeds).

The following proposition is one of the key results that makes geometry generators useful for fabric design.
\begin{proposition} \label{prop:GeometryGeneratorsAreUnbiased}
Boundary conforming geometry generators are unbiased.
\end{proposition}
\begin{proof}
Denote the generator by $\xdd + \h_2(\x, \xd) = \zero$. Let $\lambda(t) = \|\xd(t)\|$ so that $\xd(t) = \lambda(t)\,\|\widehat{\xd}(t)\|$. Since $\xd\rightarrow\zero$, $\lambda\rightarrow 0$. Moreover, since $\h_2$ is homogeneous of degree 2, $\h_2\big(\x(t), \xd(t)\big) = \lambda(t)\,\h_2\big(\x(t), \widehat{\xd}(t)\big)$. Therefore, 
\begin{align}
    \big\|\V_\infty^\tr\h_2\big(\x(t), \xd(t)\big)\big\| 
    = \lambda(t) \big\|\V_\infty^\tr \h_2\big(\x(t), \widehat{\xd}(t)\big)\big\|
    \rightarrow \zero
\end{align}
since $\big\|\V_\infty^\tr \h_2\big(\x(t), \widehat{\xd}(t)\big)\big\|$ is finite in the limit by definition of boundary conformance.



\end{proof}

\begin{corollary}
Boundary conforming Finsler energies are unbiased.
\end{corollary}
\begin{proof}
The equations of motion of a Finsler energy form a geometry generator by Theorem~\ref{thm:FinslerGeometries}. Since the energy is boundary conforming, so too is this geometry generator. Therefore, by Proposition~\ref{prop:GeometryGeneratorsAreUnbiased} the equations of motion are unbiased, so by definition the Finsler energy is unbiased.
\end{proof}

\begin{corollary}[Geometric fabrics]
Suppose $\h_2(\x, \xd)$ is homogeneous of degree 2 and unbiased so that $\xdd + \h_2(\x, \xd)  = \zero$ is an unbiased geometry generator, and suppose $\Lag_e$ is Finsler and boundary conforming. Then the energized system is fabric defined by a generator whose geometry matches the original generator's geometry. Such a fabric is called a {\em geometric fabric}.
\end{corollary}
\begin{proof}
By Corollary \ref{cor:BentFinslerRepresentation} the energized system is a generator with matching geometry, and by Theorem~\ref{thm:EnergizedFabrics} that energized system forms a fabric.
\end{proof}

\section{Fabrics on a transform tree}

Above in Theorem~\ref{thm:ClosureUnderSpecAlgebra} we showed that fabrics behave naturally under spec algebra operations in the sense that each of the above described classes is closed under these operations. Here we show additionally that the operation of energization outlined in Theorem~\ref{prop:SystemEnergization} commutes with the pullback operator as long as the pullback is performed with respect to the metric defined by the energy used for energization. 

We already know from Lagrangian mechanics that if we define an energy Lagrangian $\Lag_e(\x, \xd)$ we can either derive the Euler-Lagrange equation $\M_e\xdd + \f_e = \zero$ in $\mathcal{X}$ and pull it back to $\mathcal{Q}$ to get $\big(\J^\tr\M_e\J\big)\xdd + \J^\tr\big(\f_e - \Jd\qd\big) = \zero$ or pull the Lagrangian back to $\mathcal{Q}$ to get $\wt{\Lag}_e(\q, \qd) = \Lag_e(\phi(\q), \J\qd)$ first and apply the Euler-Lagrange equation direction to that pullback Lagrangian to get $\wt{\M}_e\qdd + \wt{\f}_e = \zero$, and the resulting equations of motion will be the same with $\wt{\M}_e = \big(\J^\tr\M_e\J\big)$ and $\wt{\f}_e = \J^\tr\big(\f_e - \Jd\qd\big)$ \cite{ratliff2020FinslerGeometry}. This standard result shows that the operation of deriving the Euler-Lagrange equation from a Lagrangian commutes with the pullback transform. We can either apply the Euler-Lagrange equation in the co-domain (ambient space) and pull back the resulting equations of motion, or pull back the Lagrangian to the domain and directly apply the Euler-Lagrange equation there. The resulting equations will match.

The result presented in Theorem~\ref{thm:EnergizationCommutesWithPullback} shows that the same type of commutivity holds for the energization operation as well. This result is specific to the case where the differentiable map defines an embedding, and the intuition comes from understanding that case as well. An example is where the differentiable map is a full-rank map from a $d$-dimensional space of generalized coordinates $\mathcal{Q}$ into a higher-dimensional ambient space $\mathcal{X}$ of dimension $n>d$. The embedded manifold may be viewed as a constraint, and the coordinates $\mathcal{Q}$ define generalized coordinates for that constraint. The theorem states that if there is an energy Lagrangian defined on the ambient space along with some differential equation $\xdd + \h(\x, \xd) = \zero$, we can either energize the ambient space and pullback the resulting equations or first pull back the differential equation {\em with respect to the energy Lagrangian's metric} and energize the equation there. The theorem shows that when we use the energization operation to define a fabric the most fundamental element is the energy metric. The energy defines how the pullback of the differential equation must be performed in order to remain consistent with the energization operation. 

We will see below in Section~\ref{prop:ExecutionEnergyRegulation} that in practice we do not need to explicitly energize. Instead, we will simply pull the differential equation back to the root with respect to the energy metric and keep track of the energy itself to define a lower bound on the damping required to maintain stability while instead controlling an alternative execution energy. The energy Lagrangian's primary role in shaping the equations of motion is to define the metric. Theorem~\ref{prop:weighted_geometries} shows that this basic theorem allows us to view energization as inducing metric weighted averages of acceleration policies in different task spaces.

\subsection{Energization commutes with pullback}

\begin{theorem} \label{thm:EnergizationCommutesWithPullback}
Let $\Lag_e$ be an energy Lagrangian, and let $\xdd + \h(\x,\xd) = \zero$ be a second-order differential equation with associated natural form spec $(\M_e, \f)$ under metric $\M_e = \partial^2_{\xd\xd}\Lag_e$ where $\f = \M_e\h$. Suppose $\x = \phi(\q)$ is a differentiable map for which the pullback metric $\J^\tr\M_e\J$ is full rank. Then
\begin{align}
    \mathrm{energize}_{\mathrm{pull}\Lag_e}\Big(
        \mathrm{pull}_\phi \big(\M_e,\f_2\big)
    \Big)
    =
    \mathrm{pull}_\phi\Big(
        \mathrm{energize}_{\Lag_e} \big(\M_e, \f_2\big)
    \Big).
\end{align}
We say that the energization operation commutes with the pullback transform.
\end{theorem}
\begin{proof}
We will show the equivalence by calculation. The energization of $\xdd + \h = \zero$ in force form $\M_e\xdd + \f = \zero$ with $\f = \M_e\h$ is $\M_e\xdd + \f_e^{\h}$ where
\begin{align}
    \f_e^{\h} = \f_e + \M_e\left[\M_e^{-1} - \frac{\xd\xd^\tr}{\xd^\tr\M_e\xd}\right] \big(\f - \f_e\big),
\end{align}
where $\f_e = \partial_{\xd\x}\Lag_e\xd - \partial_\x\Lag_e$ so that $\M_e\xdd + \f_e = \zero$ is the energy equation. Let $\J = \partial_\x\phi$. The pullback of the energized geometry generator is
\begin{align}
    &\J^\tr\M_e\left(\J\qdd + \Jd\qd\right) + \J^\tr\f_e^{\h} = \zero \\
    &\Rightarrow \big(\J^\tr\M_e\J\big)\qdd + \J^\tr\big(\f_e^{\h} + \M_e\Jd\qd\big) = \zero\\
    &\Rightarrow \big(\J^\tr\M_e\J\big)\qdd + \J^\tr \f_e
        + \J^\tr\M_e\left[\M_e^{-1} - \frac{\xd\xd^\tr}{\xd^\tr\M_e\xd}\right]
            \big(\f - \f_e\big) 
        + \J^\tr\M_e\Jd\qd = \zero \\\label{eqn:EnergizationPullback}
    &\Rightarrow \wt{\M}_e\qdd + \wt{\f}_e 
        + \J^\tr\M_e\left[\M_e^{-1} - \frac{\xd\xd^\tr}{\xd^\tr\M_e\xd}\right]
            \big(\f - \f_e\big),
\end{align}
where $\wt{\M}_e = \J^\tr\M_e\J$ and $\wt{\f}_e = \J^\tr\big(\f_e + \M_e\Jd\qd\big)$ form the standard pullback of $(\M_e, \f_e)$.

We can calculate the geometry pullback with respect to the energy metric $\M_e$ by pulling back the metric weighted force form of the geometry $\M_e\xdd + \f = \zero$, where again $\f = \M_e\h$. The pullback is
\begin{align}
    &\J^\tr\M_e\big(\J\qdd + \Jd\qd\big) + \J^\tr\f = \zero \\
    &\Rightarrow \big(\J^\tr\M_e\J\big) \qdd + \J^\tr\big(\f + \M_e\Jd\qd\big) \\\label{eqn:PullbackGeometry}
    &\Leftrightarrow \wt{\M}_e\qdd + \wt{\f} = \zero
\end{align}
where $\wt{\M}_e = \J^\tr\M_e\J$ as before and $\wt{\f} = \J^\tr \big(\f + \M_e\Jd\qd\big)$.

Let $\wt{\Lag}_e = \Lag_e\big(\phi(\q), \J\qd\big)$ be the pullback of the energy function $\Lag_e$. We know that the Euler-Lagrange equation commutes with the pullback, so applying the Euler-Lagrange equation to this pullback energy $\wt{L}_e$ is equivalent to pulling back the Euler-Lagrange equation of $\Lag_e$. This means we can calculate the Euler-Lagrange equation of $\Lag_e$ as
\begin{align}
    &\big(\J^\tr\M_e\J\big) \qdd + \J^\tr\big(\f_e + \M_e\Jd\qd\big) = \zero\\
    &\Leftrightarrow \wt{\M}_e\qdd + \wt{\f}_e = \zero,
\end{align}
with $\wt{\M}_e = \J^\tr\M_e\J$ and $\wt{\f}_e = \J^\tr\big(\f_e + \M_e\Jd\qd\big)$ (both as previously defined). Therefore, energizing \ref{eqn:PullbackGeometry} with $\wt{\Lag}_e$ gives
\begin{align}
    &\wt{\M}_e\qdd + \wt{\f}_e 
      + \wt{\M}_e\left[
        \wt{\M}_e^{-1} - \frac{\qd\qd^\tr}{\qd^T\wt{\M}_e\qd}\right]
        \big(\wt{\f} - \wt{\f}_e\big) = \zero \\
    &\Rightarrow 
    \wt{\M}_e\qdd + \wt{\f}_e
      + \big(\J^\tr\M_e\J\big)
          \left[
            \wt{\M}_e^{-1} - \frac{\qd\qd^\tr}{\qd^\tr \J^\tr\M_e\J\qd}
          \right]
        \Big(
          \J^\tr\big(\f + \M_e\Jd\qd\big)
          - \J^\tr\big(\f_e + \M_e\Jd\qd\big)
        \Big)
      = \zero \\
    &\Rightarrow
    \wt{\M}_e\qdd + \wt{\f}_e
      + \big(\J^\tr\M_e\big)
          \J\left[
            \big(\J^\tr\M_e\J\big)^{-1} - \frac{\qd\qd^\tr}{\xd^\tr\M_e\xd}
          \right]\J^\tr
        \big(\f- \f_e\big) = \zero\\\label{eqn:PullbackEnergizationAlmost}
    &\Rightarrow
    \wt{\M}_e\qdd + \wt{\f}_e
      + \J^\tr\M_e
          \left[
            \J\big(\J^\tr\M_e\J\big)^{-1}\J^\tr - \frac{\xd\xd^\tr}{\xd^\tr\M_e\xd}
          \right]
        \big(\f- \f_e\big) = \zero.
\end{align}
Since 
\begin{align}
    \J^\tr\M_e\J\big(\J^\tr\M_e\J\big)^{-1}\J^\tr
    = \J^\tr = \J^\tr\M_e\big(\M_e^{-1}\big),
\end{align}
we can write Equation~\ref{eqn:PullbackEnergizationAlmost} as
\begin{align}
    \wt{\M}_e\qdd + \wt{\f}_e
      + \J^\tr\M_e
          \left[
            \M_e^{-1} - \frac{\xd\xd^\tr}{\xd^\tr\M_e\xd}
          \right]
        \big(\f- \f_e\big) = \zero,
\end{align}
which matches the expression for the energized geometry pullback in Equation~\ref{eqn:EnergizationPullback}.
\end{proof}

Theorem~\ref{thm:EnergizationCommutesWithPullback} shows that one concise way to compute the energized geometry in the root is to first energize the leaves and then perform standard pullbacks. However, it is equally valid to simply pullback the geometries with respect to the energy metrics and energize the result. The following proposition shows that we can view such a pullback geometry as a metric weighted average of geometries.

\begin{proposition}[Metric weighted average of geometries.] \label{prop:weighted_geometries}
Let $\x_i = \phi_i(\q)$ for $i=1,\ldots,m$ denote the star-shaped reduction of any transform tree, and suppose the leaves are populated with geometries $\xdd_i + \h_{2,i} = \zero$ with Finsler energies $\Lag_{e_i}$ with energy tensors $\M_i = \partial^2_{\xd\xd}\Lag_{e_i}$. Then the metric weighted pullback of the full leaf geometry is $\qdd + \wt{\h}_2 = \zero$,
with 
\begin{align}
    \wt{\h}_2 = \left(\sum_{i=1}^m\wt{\M}_i\right)^{-1}\sum_{i=1}^m \wt{\M}_i\wt{\h}_{2,i},
\end{align}
where $\wt{\M}_i = \J^\tr\M_i\J$ and $\wt{\h}_{2,i} = \wt{\M}_i^{\dagger}\J^\tr\M_i\big(\h_{2,i} - \Jd\qd\big)$ are the standard pullback components written in acceleration form.
\end{proposition}
\begin{proof}
The standard algebra on $(\M_i\f_i)$ holds, where $\f_i = \M_i\h_{2,i}$. Pulling back gives $(\wt{\M}_i,\wt{\f}_i)$ and summing gives
$\sum_i (\wt{\M}_i,\wt{\f}_i) = \big(\sum_i\wt{\M}_i, \sum_i\wt{\f}_i\big)$.
Expressing that result in canonical form gives
\begin{align}
    \left(\sum_i\wt{\M}_i, \Big(\sum_i\wt{\M}_i\Big)^{-1}\sum_i\wt{\M}_i\wt{\h}_{2,i}\right),
\end{align}
where $\wt{\h}_{2,i}$ is the acceleration form of the individual pullbacks. Expanding gives the formula.

Since $\Lag_{e_i}$ are Finsler energies, the pullback metrics $\wt{\M}_i$ are homogeneous of degree 0 in velocity (i.e. they depend only on the normalized velocity $\hat{\xd}$). Therefore, $\wt{\h}_2$ is homogeneous of degree 2 and the pullback forms a geometry generator.
\end{proof}

\section{Closure under the spec algebra} \label{sec:ClosureTheorem}


Many of the above classes of fabric introduced above (possibly all of them, see below) are closed under the spec algebra. Specifically, we say that a class of fabrics is closed if a spec from that class remains in the same class under the spec algebra operations of combination and pullback.

\begin{theorem} \label{thm:ClosureUnderSpecAlgebra}
The following classes of fabrics are known to be closed under the spec algebra: Lagrangian fabrics, Finsler fabrics, and geometric fabrics.
A fabric of each of these types will remain a fabric of the same type under spec algebra operations in regions where the differentiable transforms are full rank and finite.
\end{theorem}
\begin{proof}
Lagrangian systems are covariant (see \cite{ratliff2020SpectralSemiSprays}) and the Euler-Lagrange equation commutes with pullback. Since full rank finite differentiable transforms define submanifold constraints and diffeomorphisms on those constraints, Lagrangian systems are in general closed under pullback. Moreover, the homogeneity of the Finsler structure is preserved under differentiable map composition since the differentiable map is independent of velocity, so Finsler systems, themselves, are closed under pullback. Both Lagrangian and Finsler systems are individually closed under Cartesian product, so since any collection of spec algebra operations can be expressed as a pullback from a Cartesian product space to a root space \cite{ratliff2020SpectralSemiSprays}, Lagrangian and Finsler systems are closed under the spec algebra.

For geometric fabrics, if we can show that geometry generators are closed under the spec algebra, then since Finsler energies are as well and geometric fabrics are defined by energization, geometric fabrics themselves are closed under the spec algebra. To examine the closure of geometry generators, we note that the pullback $\wt{\h}_2 = \J^\tr\M\big(\h_2 + \Jd\qd\big)$ of a geometry generator $\h_2(\x, \xd)$ with respect to a Finsler metric $\M(\x, \xd)$ under differentiable map $\x = \phi(\q)$ remains homogeneous of degree 2 since the Finsler metric is homogeneous of degree 0 and the curvature term $\Jd\qd = \big(\partial_{\x\x}^2\phi\,\qd\big)\qd$ is homogeneous of degree 2. Therefore, it's a geometry generator by definition.
\end{proof}

We believe (conjecture) that the broader classes of general optimization fabrics, conservative fabrics, and Finsler energized fabrics are closed under the spec algebra as well, although the proofs are more technical and we have yet to work through them in full detail. We will add a complete theorem statement and proof in a future version of this work.

\section{Speed control via execution energy regulation}
\label{sec:speed_control}

In general, once a geometry is energized with respect to some Finsler energy $\Lag_e$, it usually does not exhibit constant Euclidean speed $\|\xd\|$ as it attempts to conserve $\Lag_e$. In many cases, though, we primarily care about the shape of the behavior and would prefer to control the rate of travel through the space separately. Since the underlying geometric fabric is defined by a speed agnostic velocity, doing so it relatively straightforward. In order to maintain constant Euclidean speed, the behavioral Finsler energy $\Lag_e$ would have to increase or decrease as needed. Intuitively, increasing that energy can be done by injecting energy into the system using the potential function, and decreasing the energy can be done using damping. These two operations give us latitude to regulate execution energies (such as the Euclidean energy) while still working within the framework outlined by the geometric fabric optimization theorems to guarantee convergence and stability. We present the primary tools that enable such execution energy regulation in the following proposition.

\begin{proposition} \label{prop:ExecutionEnergyRegulation}
    Suppose $\xdd + \h_2(\x,\xd) = \zero$ is a geometry generator., $\Lag_e$ a system energy with metric $\M_e$, and $\h$ a forcing potential. Let $\alpha_{\Lag_e}$ be such that $\xdd = -\h_2 + \alpha_{\Lag_e}\xd$ maintains constant $\Lag_e$ (the energization coefficient), $\alpha_\mathrm{ex}^0$ be such that $\xdd = -\h_2 + \alpha_\mathrm{ex}^0\xd$ maintains constant execution energy $\Lag_e^\mathrm{ex}$, and $\alpha_\mathrm{ex}^\psi$ be such that $\xdd = -\h_2 - \M_e^{-1}\partial_\x\psi + \alpha_\mathrm{ex}^\psi \xd$ maintains constant execution energy. Let $\alpha_\mathrm{ex}$ be an interpolation between $\alpha_\mathrm{ex}^0$ and $\alpha_\mathrm{ex}^\psi$. Then the system
    \begin{align}
        \xdd = -\h_2 - \M_e^{-1}\partial_\x\psi + \alpha_\mathrm{ex}\xd - \beta\xd
    \end{align}
    is optimizing when $\beta > \alpha_\mathrm{ex} - \alpha_{\Lag_e}$.
\end{proposition}
\begin{proof}
By definition $\alpha_{\Lag_e}$ is an energization coefficient, so 
\begin{align}
    \xdd = -\h_2 - \M_e^{-1}\partial_\x\psi + \alpha_\mathrm{ex}\xd - \wt{\beta}\xd
\end{align}
optimizes when $\wt{\beta} > 0$. That means
\begin{align}
    \xdd 
    &= -\h_2 - \M_e^{-1}\partial_\x\psi + \big(\alpha_{\Lag_e} + \alpha_\mathrm{ex} - \alpha_\mathrm{ex}\big)\xd - \wt{\beta}\xd\\
    &= -\h_2 - \M_e^{-1}\partial_\x\psi + \alpha_\mathrm{ex}\xd - \big(\alpha_\mathrm{ex} - \alpha_{\Lag_e} - \wt{\beta}\big)\xd
\end{align}
optimizes when $\wt{\beta}>0$. Using $\beta=\alpha_\mathrm{ex} - \alpha_{\Lag_e} + \wt{\beta}$ we have $\beta - (\alpha_\mathrm{ex} - \alpha_{\Lag_e}) = \wt{\beta}>0$ which implies $\beta>\alpha_\mathrm{ex} - \alpha_{\Lag_e}$.
\end{proof}

The fundamental requirement from the analysis to maintain constant $\Lag_e$ is  $\beta = \alpha_\mathrm{ex} - \alpha_{\Lag_e}$; to optimize, we must ensure that $\beta$ is strictly larger than that (in particular, in order to converge to a local minimum). Often we will additionally use the restriction $\beta\geq 0$ to maintain the semantics of a damper, giving the following inequality $\beta \geq \max\{0, \alpha_\mathrm{ex} - \alpha_{\Lag_e}\}$. To converge, that inequality must be strict.

Note that $\alpha_\mathrm{ex}^0$ is the standard energy transform for the geometric term $-\h_2$ alone, while $\alpha_\mathrm{ex}^\psi$ ensures that the forcing term is included in the transform as well. Therefore, if $\beta=0$ we have the following:

The system under $\alpha_\mathrm{ex}^\psi$
\begin{align}\label{eqn:SpeedRegulation1}
    \xdd_1 = -\h_2 - \M_e^{-1}\partial_\x\psi + \alpha_\mathrm{ex}^\psi\xd
\end{align}
will maintain constant execution energy $\Lag_e^\mathrm{ex}$ while still being forced. 
Likewise, under $\alpha_\mathrm{ex}^0$ and zero potential $\psi = \zero$, the system 
\begin{align}\label{eqn:SpeedRegulation2}
    \xdd_2 = -\h_2 + \alpha_\mathrm{ex}^0\xd
\end{align}
will maintain constant $\Lag_e^\mathrm{ex}$ while the forced system 
\begin{align}\label{eqn:SpeedRegulation3}
    \xdd_3 = -\h_2 - \M_e^{-1}\partial_\x\psi + \alpha_\mathrm{ex}^0\xd
\end{align}
will force the system while moving between energy levels as well. Therefore, the difference between Equations~\ref{eqn:SpeedRegulation3} and \ref{eqn:SpeedRegulation1} must be the extra component of $-\M_e^{-1}\partial_\x\h$ accelerating the system with respect to this execution energy. That component is
\begin{align}
    \xdd_3 - \xdd_1
    &= \Big(-\h_2 - \M_e^{-1}\partial_\x\psi + \alpha_\mathrm{ex}^0\xd\Big)
      - \Big(-\h_2 - \M_e^{-1}\partial_\x\psi + \alpha_\mathrm{ex}^\psi\xd\Big)\\
    &= \big(\alpha_\mathrm{ex}^0 - \alpha_\mathrm{ex}^\psi\big)\xd.
\end{align}
An interpolation between Equations~\ref{eqn:SpeedRegulation3} and \ref{eqn:SpeedRegulation1} gives a scaling to this component: Let $\eta\in[0,1]$ then
\begin{align}
    \eta\xdd_3 + (1-\eta)\xdd_1
    &= \eta \Big(
            -\h_2 - \M_e^{-1}\partial_\x\psi 
            + \alpha_\mathrm{ex}^0\xd\Big)
        - (1-\eta)\Big(
            -\h_2 - \M_e^{-1}\partial_\x\psi 
            + \alpha_\mathrm{ex}^\psi\xd\Big) \\
    &= -\h_2 - \M_e^{-1}\partial_\x\psi 
        + \left(
            \eta \alpha_\mathrm{ex}^0 
            + (1-\eta) \alpha_\mathrm{ex}^\psi
        \right)\xd.
\end{align}
The added component is now
\begin{align}
    &\Big(\eta\xdd_3 + (1-\eta)\xdd_1\Big) - \xdd_1 \\
    &\ \ \ \ = -\h_2 - \M_e^{-1}\partial_\x\psi 
        + \left(
            \eta \alpha_\mathrm{ex}^0 
            + (1-\eta) \alpha_\mathrm{ex}^\psi
        \right)\xd
        - \left(-\h_2 - \M_e^{-1}\partial_\x\psi 
            + \alpha_\mathrm{ex}^\psi\xd\right) \\
    &\ \ \ \ = \left[\eta \alpha_\mathrm{ex}^0 + (1-\eta)\alpha_\mathrm{ex}^\psi - \alpha_\mathrm{ex}^\psi\right]\xd \\
    &\ \ \ \ = \eta\big(\alpha_\mathrm{ex}^0 - \alpha_\mathrm{ex}^\psi\big)\xd.
\end{align}
when $\eta = 0$ we take the entirety of system $\xdd_1$, which projects the entirety of $-\h_2 - \M_e^{-1}\partial_\x\psi$. In correspondence, $\eta \big(\alpha_\mathrm{ex}^0 - \alpha_\mathrm{ex}^\psi\big)\xd = 0$ when $\eta = 0$. Similarly, when $\eta = 1$, we take all of $\xdd_3$ which leaves the full $\M_e^{-1}\partial_\x\psi$ in tact. Here, the entire component $\eta \big(\alpha_\mathrm{ex}^0 - \alpha_\mathrm{ex}^\psi\big)\xd$ remains in tact when $\eta = 1$.

Thus, the interpolated coefficient $\alpha_\mathrm{ex} = \eta \alpha_\mathrm{ex}^0 + (1-\eta) \alpha_\mathrm{ex}^\psi$ referenced in the theorem acts to modulate the component $\eta \big(\alpha_\mathrm{ex}^0 - \alpha_\mathrm{ex}^\psi\big)\xd$ defining the amount of $-\M_e^{-1}\partial_\x\psi$ to let through to move the system between execution energy levels. A typical strategy for speed control could then be to:
\begin{enumerate}
    \item Choose an execution energy $\Lag_e^\mathrm{ex}$ to modulate.
    \item At each cycle, calculate $\alpha_\mathrm{ex}^0$, $\alpha_\mathrm{ex}^\psi$, $\alpha_{\Lag_e}$.
    \item Choose $\eta\in[0,1]$ to increase the energy as needed, using the extremes of $\eta=0$ to maintain execution energy and $\eta=1$ to fully increase execution energy.
    \item Choose damper under the constraint $\beta \geq \max\{0, \alpha_\mathrm{ex} - \alpha_{\Lag_e}\}$ with $\alpha_\mathrm{ex} = \eta \alpha_\mathrm{ex}^0 + (1-\eta) \alpha_\mathrm{ex}^\psi$. Use a strict inequality to remove energy from the system to ensure convergence. Note that even with $\eta=1$ (fully active potential), the bound will adjust accordingly and ensure convergence under strict inequality.
\end{enumerate}

The full system executed at each cycle is
\begin{align}
    \xdd = \vv_\perp + \eta \vv_\paral - \beta \xd,
\end{align}
(with $\eta$ and $\beta$ defined as above) where $\vv_\perp = \mP_{\Lag_e^\mathrm{ex}}\big[-\h_2 - \M_e^{-1}\partial_\x\psi\big] = -\h_2 - \M_e^{-1}\partial_\x\psi + \alpha_\mathrm{ex}^\psi\xd$ is the component of the system preserving execution energy $\Lag_e^\mathrm{ex}$, and $\vv_\paral$ is the remaining (execution energy changing) component such that $\vv_\perp + \vv_\paral = -\h_2 - \M_e^{-1}\partial_\x\psi$ reconstructs the original system.

When $\beta = \alpha_\mathrm{ex} - \alpha_{\Lag_e}$, the underlying behavioral energy $\Lag_e$ is strictly maintained. When the damping is strictly larger than that lower bound $\beta > \alpha_\mathrm{ex} - \alpha_{\Lag_e}$, that system energy decreases. As long as that strict inequality is satisfied, convergence is guaranteed. 

The projection equation for $\alpha$ is calculated in the same way as the energization alpha used to energize geometries:
\begin{align}
    \alpha = -\big(\xd^\tr\M_e\xd\big)^{-1}\xd^\tr\Big[\M_e\xdd_d - \f_e\Big],
\end{align}
where the spec $(\M_e,\f_e)$ define the energy equation $\M_e\xdd + \f_e = \zero$ for the underlying behavioral energy $\Lag_e$ and $\xdd_d$ in this case is either $\xdd_d^0 = -\h_2$ for $\alpha_\mathrm{ex}^0$ or $\xdd_d^\psi = -\h_2 - \M_e^{-1}\partial_\x\psi$ for $\alpha_\mathrm{ex}^\psi$.

Note that for Euclidean $\Lag_e^\mathrm{ex} = \frac{1}{2}\|\xd\|^2$ we have $\M_e = \I$ and $\f_e = \zero$, so
\begin{align}
    \alpha_\mathrm{ex} = -\big(\xd^\tr\xd\big)^{-1}\xd^\tr\Big[\xdd_d - \zero\Big]
    = \frac{-\xd^\tr\xdd_d}{\xd^\tr\xd},
\end{align}
which gives
\begin{align}
    \xdd 
    &= \xdd_d + \alpha_\mathrm{ex}\xd \\
    &= \xdd_d - \frac{-\xd^\tr\xdd_d}{\xd^\tr\xd} \xd
    = \Big[\I - \hat{\xd}\hat{\xd}^\tr\Big]\xdd_d \\
    &= \mP_\xd^\perp\big[\xdd_d\big],
\end{align}
where $\mP_\xd^\perp$ is the projection matrix projecting orthogonally to $\xd$. The above analysis is therefore just a generalization of this form of orthogonal projection to arbitrary execution energies.



\section{Behavior Design with Geometric Fabrics} \label{sec:DesignWithGeometricFabrics}



This section describes some specific mathematical tools we use in practice (including in our experiments) to construct geometric fabrics.

\subsection{Construction of Geometries}
In general, geometries are functions of position and velocity and the only requirement is that the function is homogeneous of degree 2 in velocity. However, strategies for easily constructing geometries exist. For instance, a formulaic approach to constructing a geometry is to design a potential function in task space $\x$ as $\psi(\x)$. Taking the gradient of this potential with respect to position and pre-multiplying by the scaled inner product of $\xd$ will create a geometry of the form,

\begin{align}
    \h_{2}(\x, \xd) = \lambda \|\xd\|^2 \partial_\x \psi(\x).
\end{align}

This geometry will effectively create a force in the negative direction of the potential gradient, scaled by velocity and $\lambda \in \mathbb{R}^+$. The homogeneity condition is easily met with the inclusion of $\|\xd\|^2$, i.e., $\|\lambda \xd\|^2 = (\lambda\xd)^\tr(\lambda\xd) = \lambda^2  \xd^\tr\xd = \lambda^2\|\xd\|^2$. Overall, geometric behavior with this design generates forces that push towards the minima of $\psi(\x)$. However, this design is not the only way to build geometries, and others can be constructed without any concept of pushing towards potential minima.

\subsection{Prioritization of Geometries}
As discussed in Proposition \ref{prop:weighted_geometries}, the extent to which geometries define system behavior can be modulated through the pulled back energy tensors, $\wt{\M}_i$, of Finsler energies, $\Lag_{e,i}$. These energy tensors capture priority in both direction and magnitude for the associated pulled-back geometry, $\wt{\h}_{2,i}$. Therefore, in a collection of weighted geometries, 

\begin{align}
    \wt{\h}_2 = \left(\sum_{i=1}^m\wt{\M}_i\right)^{-1}\sum_{i=1}^m \wt{\M}_i\wt{\h}_{2,i},
\end{align}

\noindent individual geometries can dominate the total system geometry, $\wt{\h}_2$, under certain conditions and play a rather minor role under other conditions. The design of $\wt{\M}_i$ is rather open ended, but can easily be based on intuition of when specific geometries should have priority. For instance, geometries that are constructed on one-dimensional distance task spaces, $x \in \mathbb{R}$, for obstacle avoidance can use a barrier-type potential where $\psi(x) \to \infty$ as $x \to 0$ and $\psi(x) \to 0$ as $x \to \infty$. The energy, $\Lag_{e,i}$, can therefore be designed as $\Lag_{e,i} = \frac{1}{2} G(x) \dot{x}^2$, where $
M_i(x) = G_i(x)$ since $G_i(x)$ does not have velocity dependence. The metric $G_i(x)$ can be chosen as $G_i(x) = \psi(x)$ which means that priority for the associated geometry $h_{2,i}$ would increase as distance to the object decreased. Moreover, priority would vanish as distance to the object decreased. This prioritization intuitively makes sense because obstacle avoidance behavior should dominate the system when close to the object.

\subsection{Acceleration-based Potential Design}
\label{subsec:acceleration_potential}

Although the theory provides that any potential function can be optimized over an arbitrary geometric fabric, the shape of the potential function can lead to more desirable system behavior and timely convergence. Specifically, an acceleration-based potential design is easy to use, tune, and starts with the desire to accelerate along the negative gradient of a baseline potential $\psi_1(\q)$,

\begin{align}
    \qdd = -\partial_\q \psi_1(\q).
\end{align}

\noindent Next, this acceleration profile is prioritized with the energy tensor, $\M_\psi$, from the Finsler energy, $\Lag_{e,\psi} = \q^T \G(\q) \q$ (see Appendix D.4 in \cite{cheng2018rmpflowarxiv} for a related discussion on metric compatibility for acceleration-based (motion policy) design) resulting in 

\begin{align}
   \M_\psi \qdd = -\M_\psi \partial_\q \psi_1(\q).
\end{align}

\noindent Importantly, $\Lag_{e,\psi}$ should be added to a system's energy such that the system mass includes $\M_\psi$, thereby, weighing this policy just like other components (geometries) within the system. This prioritization requires that there exists a total potential, $\psi(\q)$, with an associated gradient,

\begin{align}
    \partial_\q \psi(\q) = \M_\psi \partial_\q \psi_1(\q).
\end{align}

 \noindent For this form of $\partial_\q \psi(\q)$ to originate from a valid, scalar potential function, $\psi_1(\q)$ and $\M_e$ must be chosen such that $\partial_{\q \q}^2 \psi(\q)$ is symmetric. This property is facilitated by designing $\M_{\psi} = w(\|x\|^2)I$, where $w(\cdot) \in \mathbb{R}^+$ is a scalar function that makes $\M_{\psi}$ radially symmetric. Similarly, let $\psi_1(x) = l(\|x\|^2)$ be a radially symmetric potential function. Symmetry of $\partial_{\q \q}^2 \psi(\q)$ can then be shown as

\begin{align}
    \partial_{\q \q}^2 \psi(\q)
    &= \partial_\q \left[ (w(\|\q\|^2)I) (2l' (\|\q\|^2) \q) \right]\\
    &= \partial_\q \left[ r (\|\q\|^2) \q \right]\\
    &= r(\|\q\|^2)I + 2r' (\|\q\|^2) \q \q^T
\end{align}

\noindent where $r=2w(\|\q\|^2) l'(\|\q\|^2)$. The final expression is a sum of symmetric terms and therefore, $\partial_{\q \q}^2 \psi(\q)$ is symmetric.

\subsection{Damping}
\label{subsec:damping}
Damping can be introduced to the system as follows. Under regulating a system's execution energy as discussed in Section \ref{sec:speed_control}, damping must satisfy  $\beta \geq \max\{0, \alpha_\mathrm{ex} - \alpha_{\Lag_e}\}$. This inequality is enforced along with timely convergence with damping designed as 

\begin{align}
\beta = s_\beta(\x) B + \underline{B} + \max\{0, \alpha_\mathrm{ex} - \alpha_{\Lag_e}\}
\end{align}

\noindent where $B$, $\underline{B} \in \mathbb{R}^+$ are a damping constant and minimum damping constant, respectively, and $s_\beta (\x)$ is a switching function. $\underline{B}$ ensures that damping always exists on the system, thereby guaranteeing progression towards convergence. However, $\underline{B}$ should be small such that energy is not significantly siphoned from the system, particularly when the optimization potential is not near its minimum. $B$ is additional damping that gets engaged when $s_\beta (\x)$ is high, ensuring timely convergence. $s_\beta (\x)$ is low when the optimizing potential is far from its minimum and the switch is high when closer. Importantly, the region for which the switching function is high should be aligned with the region for which the potential priority, $\M_\psi$, is high. Otherwise, significant damping can occur prior to heightened priority of the optimization potential leading to slow convergence to the target. Note that when not regulating a execution energy that differs from the system energy, $\alpha_\mathrm{ex} = \alpha_{\Lag_e}$. Therefore, the above damping equation degenerates to $\beta = s_\beta(\x) B + \underline{B}$.

\section{Experimental results} \label{sec:Experiments}
The following section provides empirical evidence confirming fundamental properties of geometries, their energization, and optimization over their fabrics. Moreover, canonical geometric behaviors are sequenced not only to illustrate useful designs and their effect during optimization but also to demonstrate the ease of creating increasingly complex system behavior.

We explore two types of systems. The first system is a simple 2D point mass system which we use as a controlled setting for validating the theoretical properties derived above. This series of experiments confirms many of the theoretical properties outlined in this paper, including path consistency of geometries (Section~\ref{sec:PathConsistencyExperiment}), that energization commutes with pullback (Section~\ref{sec:EnergizationCommutesExperiment}), and that we can build systems up layer-by-layer and optimize successfully at each step even with random forcing geometries (Section~\ref{sec:ComposabilityAndOptimizationExperiment}).

The second system is a 3D planar manipulator arm. This system builds off the 2D point mass setting but introduces kinematic redundancy and a separate end-effector space. We show with these planar arm experiments that the theory can be applied to practical systems to design manipulation systems including nuanced end-effector behavior not traditionally addressed by local reactive systems. These experiments include continuous point-to-point reaching with redundancy resolution (Section~\ref{sec:PointToPointReachingExperiment}) and end-effector behavior shaping (Section~\ref{sec:EndEffectorBehaviorExperiment}).

\subsection{Path Consistency of Geometries} \label{sec:PathConsistencyExperiment}
To illustrate that geometries under the generator equation, $\xdd + \h_2(\q, \qd) = 0$, produce consistent paths given initial conditions, $(\q_0, \alpha \hat{\qd}_0)$, a set of obstacle avoiding geometries are designed and integrated forward. For each designed geometry, two different values of $\alpha$ are used to show that the yielded paths are exactly the same. The results also show the flexibility in defining geometries which can be leveraged to enable desired behavior.

A geometry $\h_2(\q, \qd)$ that naturally produces particle paths that avoid obstacles can be defined in coordinates, $\q \in \mathbb{R}^{2}$, as 

\begin{align} \label{eqn: phi2_case1}
    \h_2(\q,\qd) = \lambda \| \qd \| ^2 \: \partial_\q \psi(\phi(\q))
\end{align}

\noindent where $\phi(\q)$ is a differentiable map that captures the distance to a circular object and $\lambda \in \mathbb{R}^+$ is a scaling gain. More specifically, $\phi(\q) = \frac{\|\q - \q_o\| - r}{r}$, where $\q_o$ and $r$ are the circle's center and radius, respectively. Furthermore, $\psi(\phi(\q)) \in \mathbb{R}^+$ is a barrier potential function, $\psi(\phi(\q)) = \frac{k}{\phi(\q)^2}$, where $k \in \mathbb{R}^+$ is a scaling gain. Altogether, $\partial_\q \psi(\phi(\q))$ produces an increasing repulsive force as distance to the object decreases, and $\| \qd \| ^2$ makes $\h_2(\q, \qd)$ homogeneous of degree 2 in $\qd$. For this experiment, $\lambda = 0.7$, $k = 0.5$ for two scenarios: 1) $\alpha = 1.5$, and 2) $\alpha = 0.75$. The behavior of this geometry can be seen in Figure \ref{fig:phi2_case1}a. Noticeably, the paths generated are completely overlapping which confirms path consistency.

Another example of a generator that can produce obstacle avoiding paths is one derived from an energy CHOMP-like Lagrangian \cite{ratliff2009chomp}, $\Lag_e = \frac{1}{2} \psi(\phi(\q)) \|\qd\|^2$ where $\psi(\phi(\q))$ is the same as before. Passing this Lagrangian through the Euler-Lagrange equation produces a corresponding energy tensor and force, $\M_e$ and $\f_e$. These terms can be used to construct an obstacle avoiding geometry as

\begin{align}
    \h_2(\q,\qd) = -\M_e^{-1} \f_e.
\end{align}

\noindent For this experiment, the two scenarios are with $\alpha = 0.75$ and $\alpha = 0.375$. The behavior of this geometry can be seen in Fig. \ref{fig:phi2_case1}b. Again, path consistency among the obstacle avoiding paths is observed.

A final obstacle-avoiding geometry is constructed by leveraging a Finsler energy. Given a Finsler energy, $\Lag_e = \frac{1}{2 \phi(\q)^2}  (\J(\q) \qd)^2$, a geometry can be constructed as

\begin{align}
    \h_2(\q,\qd) = \lambda \Lag_e \: \partial_\q \psi(\phi(\q))
\end{align}

\noindent where $\phi(q)$ is a differentiable map. This geometry was simulated with $\lambda = 0.7$ for $\alpha = 1.5$ and $\alpha = 0.75$. The paths generated can be seen in Figure \ref{fig:phi2_case1}c. Again, the paths are consistent and this geometry produces an obstacle avoiding behavior. Overall, these three examples illustrate not only path consistency but also the flexibility by which geometries can be designed.

\begin{figure}
  \vspace{2mm}
  \centering
  \includegraphics[width=1.\linewidth]{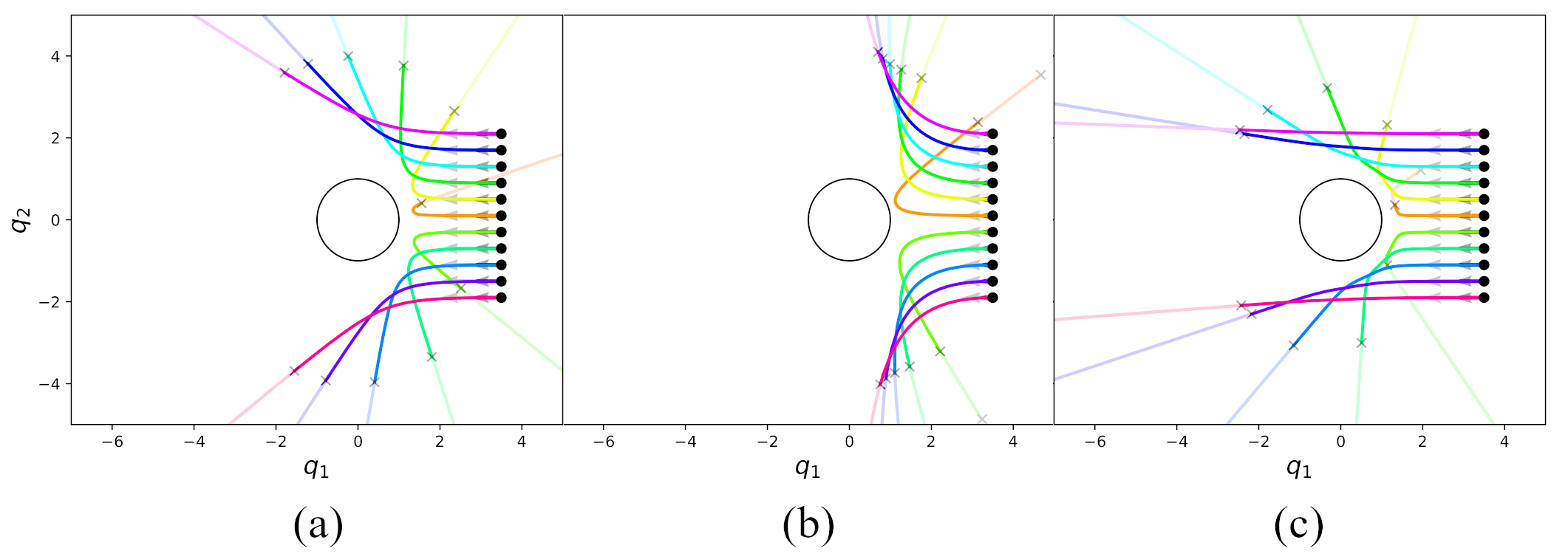}
  \vspace{-4mm}
  \caption{Paths generated from particle motion initiated at different speeds and three different generators. Faded paths indicate higher $\alpha$ and opaque paths indicate lower $\alpha$. Experiment is described in Section \ref{sec:PathConsistencyExperiment}.}
  \label{fig:phi2_case1}
  \vspace{-4mm}
\end{figure}

\subsection{Energization commutes with pullback} \label{sec:EnergizationCommutesExperiment}
Simulations are conducted to empirically demonstrate the theoretical finding that energization cummutes with pullback for full rank systems. Again, the significance of this finding is that leaf generators can be energized in their leaf coordinates and pulled back or pulled back and energized in root coordinates with an equivalent pulled back energy.

A generator is designed in polar coordinates $\x = [r, \theta ]^T$ for $\ddot{\x} + \h_2(\x, \xd) = 0$ as
\begin{align}
\h_2(\x, \xd) = \| \xd \| ^2 \partial_\x \psi(\x)
\end{align}

\noindent where $\psi(\x) = -\frac{1}{2} \| \x \| ^2$. Effectively, this generator tries to maximize the value of $\| \x \| ^2$ so the trajectories will move such that $r \to \infty$ while $\theta \to \pm \frac{\pi}{2}$. This geometry is separately energized in both the leaf and the root with $\Lag_e = \frac{1}{2} \xd^T \xd$ and produces the exact same behavior in Figure \ref{fig:energization_commutes_2D}. This result empirically confirms that energization commutes with pullback for full rank systems.

\begin{figure}
  \vspace{2mm}
  \centering
  \includegraphics[width=0.8\linewidth]{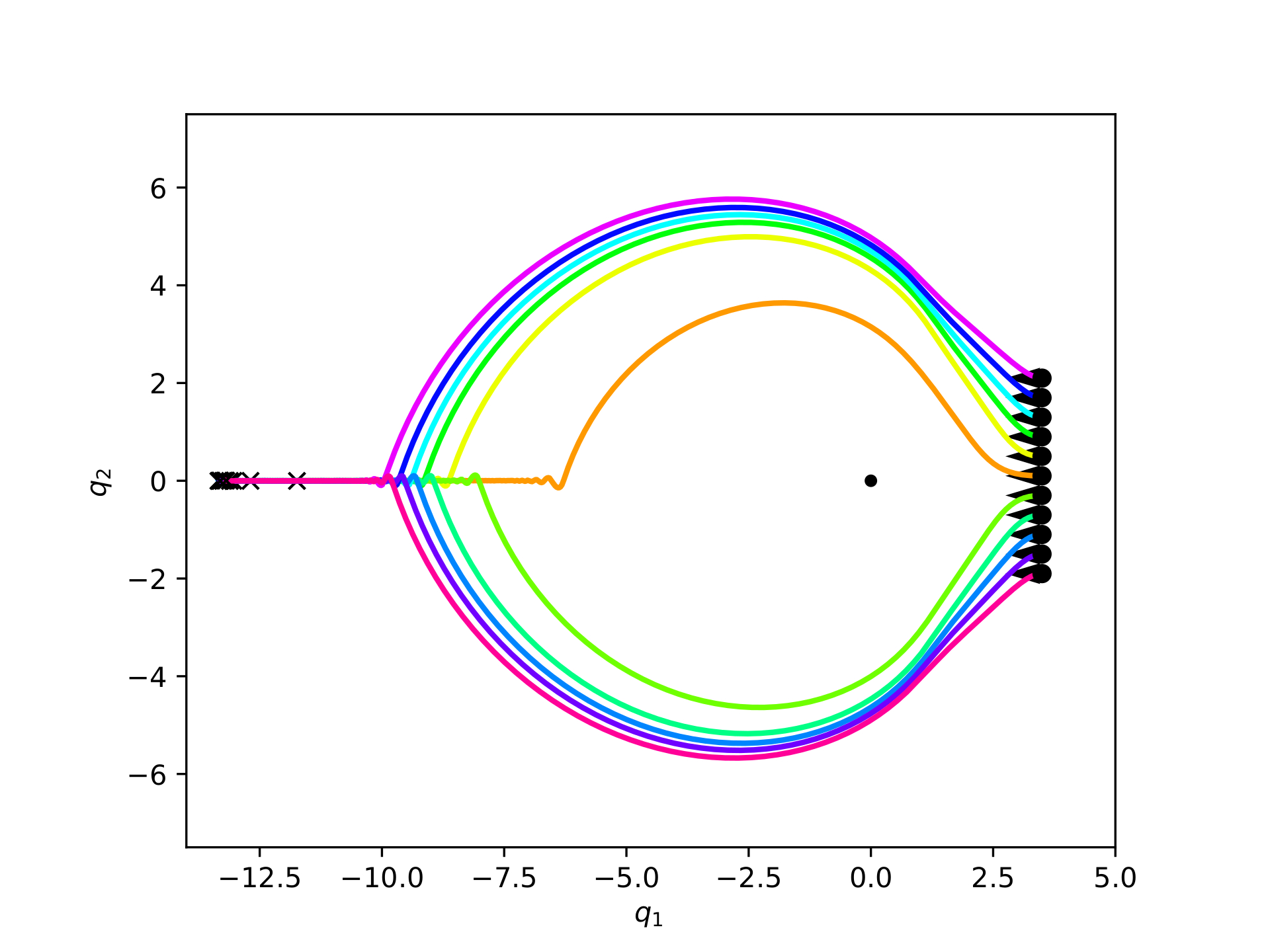}
  \vspace{-4mm}
  \caption{Paths created by a generator and energy built in polar coordinates and pulled back into root coordinates. This result is exactly the same for pulling back the generator first and then energizing in the root. Experiment is described in Section \ref{sec:EnergizationCommutesExperiment}. }
  \label{fig:energization_commutes_2D}
  \vspace{-4mm}
\end{figure}

\subsection{System Compositionality and Optimization} \label{sec:ComposabilityAndOptimizationExperiment}
As previously discussed, geometries define underlying behavior and one can conceivably arrive at increasingly richer behavior by composing multiple, individual geometries together as in Proposition \ref{prop:weighted_geometries}. Compositionality, and in particular, \textit{sequential} compositionality, of geometries is a powerful paradigm for building system behavior. For instance, desirable system behavior may include respecting coordinate boundaries, obstacle avoidance, motion bias towards a target location, and ultimately, reaching that target location. Most of these behaviors can be captured as geometries. When energized, the geometries in combination with an attraction potential and damping will achieve the desired system behavior. As subsequently demonstrated, system behavior will be incrementally advanced through the sequential addition of geometrically defined behavior. Numerical results will reveal overall behavior of the underlying geometry and the departure from these natural paths with optimization. Moreover, results will demonstrate increasingly complex system behavior through the addition of geometries and provide empirical evidence that optimization can occur over arbitrary geometries.

\subsubsection{Optimization Potential, Speed Control, and Damping Settings} 
Specific optimization and numerical integration settings that persist throughout the remaining experiments in this section include the following. A set of 14 particles were given initial conditions of $\q_0 = [2, 3]^T$ and $\|\qd_0\| = 1.5$ pointing radially outward and rotationally distributed from $0$ to $2\pi$. Forward integration of system behavior is conducted with the Runge-Kutta fourth-order routine with a maximum step size of 0.01 s for 16 s of simulation time. The optimization potential is an acceleration-based design as discussed in Section \ref{subsec:acceleration_potential} and is designed for the task space

\begin{align}
\label{eq:attractor_task_map}
\x = \phi(\q) = \q - \q_d,
\end{align}

\noindent where $\q_d$ is the desired position, in this case, $\q_d = [-2.5, -3.75]^T$. The acceleration-based potential gradient, $\partial_\q \psi(\phi(\q)) = M_\psi(\phi(\q)) \partial_\q \psi_1(\phi(\q))$, is designed with

\begin{align}
\label{eq:attractor_potential}
\psi_1(\phi(\q)) = k \left( \|\phi(\q)\| + \frac{1}{\alpha_\psi}\log(1 + e^{-2\alpha_\psi \|\phi(\q)\|} \right)
\end{align}

\noindent and

\begin{align}
\label{eq:potential_priority}
M_\psi(\phi(\q)) = (\widebar{m} - \underline{m}) e^{-(\alpha_m \|\phi(\q)\|) ^ 2} I + \underline{m} I.
\end{align}

\noindent Furthermore, $k \in \mathbb{R}^+$ is a parameter that controls the overall gradient strength, $\alpha_\psi \in \mathbb{R}^+$ is a parameter that controls the transition rate of $\partial_\q \psi_1(\phi(\q))$ from being nearly a constant to 0. $\widebar{m}$, $\underline{m} \in \mathbb{R}^+$ are the upper and lower isotropic masses, respectively. $\alpha_m \in \mathbb{R}^+$ controls the width of the radial basis function. For the following experiments, $k = 5$, $\alpha_\psi = 10$, $\widebar{m} = 2$, $\underline{m}=0.3$, and $\alpha_m = 0.75$. Importantly, the energy $\Lag_{e,\psi} = \q^T \M_\psi(\phi(\q)) \q $ is added to the system energy, $\Lag_e$. Altogether, this design allows for $\partial_\q \psi(\x)$ to be small far away from the object while smoothly escalating with heightening priority as position towards the target location decreases.

For these experiments, damping is invoked as discussed in Section \ref{subsec:damping}. Here, $B=6.5$, $\underline{B} = 0.01$, and the switching function is designed as

\begin{align}
s_\beta (\x) = \frac{1}{2} (\tanh(-\alpha_\beta (\|\x\| - r) ) + 1) 
\end{align}

\noindent where $\alpha_\beta \in \mathbb{R}^+$ defines the quickness of the switch, and $r \in \mathbb{R}^+$ defines the radius at which the switch is half-way engaged. For these experiments, $\alpha_\beta = 0.5$ and $r = 1.5$.

Finally, the switching function $\eta$ from Section \ref{sec:speed_control} that allows to increase system energy for execution energy control is designed as 

\begin{align}
\eta = \frac{1}{2} (\tanh(-\alpha_\eta (\Lag_{ex} - \Lag_{ex,d}) - \alpha_{shift} ) + 1) 
\end{align}

\noindent where $\alpha_\eta$ $\alpha_{shift} \in \mathbb{R}^+$ are parameters that control the quickness of the switch and the horizontal location of the switch, respectively. Moreover, $\Lag_{ex,d}$ is the desired execution energy level. In these experiments, controlling the Euclidean speed, $\|\qd\|$, is desired. Therefore, with a desired Euclidean speed, $v_d$, the execution energy is designed as $\Lag_{ex,d} = \frac{1}{v_d} \qd^T \qd$. Finally, $v_d = 2$ in the following experiments, which is higher than the initialized speed, $\|\qd_0\| = 1.5$.

\subsubsection{Baseline Geometry and Goal-Reaching Optimization.} \label{sec:BaselineOptimizationExperiment}
An advisable starting point for any system built from geometries is to create a baseline geometry with a non-zero energy tensor (priority) everywhere. Effectively, this ensures that the system always has mass even if other geometries are designed such that their priorities (mass) vanish to zero under desirable conditions. 

The baseline geometry is constructed from $\h_{2,b} = 0$ and prioritized with an energy tensor, $M_b$, from the Finsler energy, $\Lag_{e,b} = \frac{\lambda_b}{2} \q^T \q$, where $\lambda \in \mathbb{R}^+$ is a parameter that defines the amount of baseline inertia. For these experiments, $\lambda_b = 1$. Therefore, the weighted system geometry and system energy is

\begin{align}
\M_e \qdd = \M_b \qdd = \M_b \h_{2,b}
\end{align}

\noindent and

\begin{align}
\Lag_{e} =  \Lag_{e,b},
\end{align}

\noindent respectively. System behavior can be seen in the first row of Figure \ref{fig:geometries_optimization}. As shown, the underlying geometry results in straight-line motion due to the nature of the Euclidean geometry. During optimization, the motion of particles are redirected from their natural paths resulting in a symmetric plume that terminates at the desired location. Along these paths, particle motion quickly achieves and maintains the desired speed level before the escalation in damping removes energy from the system near the desired location, resulting in convergence to the desired target.


\subsubsection{Limit Avoidance and Goal-Reaching Optimization} \label{sec:LimitOptimizationExperiment}
A useful task space is often one-dimensional and captures the distance between a current position and some boundary. This task space is generally denoted as $x \in \mathbb{R}$ and subsequently used to describe different one-dimensional distance task spaces. For instance, the distance to an upper and lower coordinate limit can be described as

\begin{equation}
\label{eq:upper_limit_map}
x = \phi(q_j) = \bar{q}_j - q_j
\end{equation}

\noindent and 

\begin{equation}
\label{eq:lower_limit_map}
x = \phi(q_j) = q_j - \underline{q}_j,
\end{equation}

\noindent respectively, where $\bar{q}_j$ and $\underline{q}_j$ are the upper and lower limits of the $j^{th}$ coordinate of $\q$. Overall, $2n$ task maps, $x$, are needed for $n$-dimensional coordinates, $\q$. Repulsive geometries can be constructed in this one-dimensional space as

\begin{align}
\label{eq:limit_geometry}
h_{2,i}(x, \dot{x}) = \lambda \dot{x}^2 \partial_{x} \psi(x)
\end{align}

\noindent where the potential $\psi(x)$ is designed such that $\psi(x) \to \infty$ as $x \to 0$ and $\psi(x) \to 0$ as $x \to \infty$. This potential function is specifically designed here as,

\begin{align}
\psi(x) = \frac{\alpha_1}{x^2} + \alpha_2 \log \left(e^{-\alpha_3(x - \alpha_4)} + 1 \right)
\end{align}

\newgeometry{top=0.cm, bottom=0.cm, left=0.5cm, right=0.5cm}
\begin{figure}
  \centering
  \includegraphics[height=1.25\linewidth]{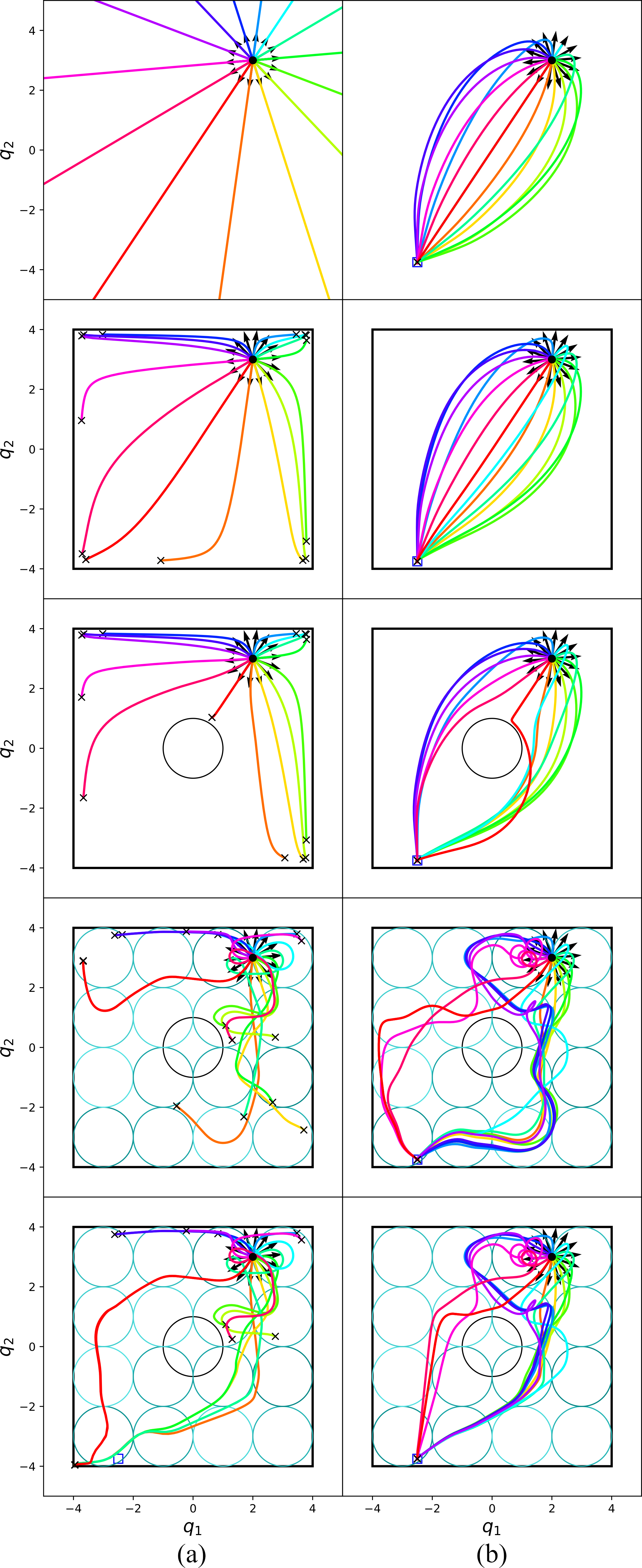}
  \vspace{-4mm}
  \caption{\small The left column (a) shows the nominal behavior of the underlying geometric fabric and the right column (b) shows the behavior optimization over it (see Sections \ref{sec:BaselineOptimizationExperiment} - \ref{sec:AttractorOptimizationExperiment} for details on each row). As the fabric complexity grows row-by-row the underlying nominal behavior becomes more sophisticated (left), but optimization (right) remains successful and consistent.}
  \label{fig:geometries_optimization}
\end{figure}
\restoregeometry

Moreover, $\alpha_1$, $\alpha_2$, $\alpha_3$, $\alpha_4 \in \mathbb{R}^+$ where, $\alpha_1$ and $\alpha_2$ are gains that control the significance and mutual balance of the first and second terms. $\alpha_3$ controls the sharpness of the smooth rectified linear unit (SmoothReLU) while $\alpha_4$ offsets the SmoothReLU. For the following experiments, $\alpha_1 = 0.4$, $\alpha_2 = 0.2$, $\alpha_3 = 20$, and $\alpha_4 = 5$. This potential function offers two key features: 1) a nearly constant gradient force to influence motion farther away from coordinate limits, and 2) an unlimited gradient force as distance to a coordinate limit shrinks, preventing motion from crossing the limit.  

To establish appropriate priorities for these geometries, individual energy tensors, $M_i$, are created from energies, $\Lag_{e_i}$, where $i=1,...,2n$. $\Lag_{e_i}$ is designed as

\begin{align}
\Lag_{e_i} = \frac{1}{2} s(\hat{\dot{x}}) \: G(x) \dot{x}^2
\end{align}

\noindent where the metric is designed as $G(x) = \frac{\lambda}{x}$ and a switching function $s(\hat{\dot{x}})$ designed as

\begin{equation}
s(\hat{\dot{x}}) = 
    \begin{cases} 
      0 & \hat{\dot{\phi}}(q_j) \geq 0 \\
      1 & \hat{\dot{\phi}}(q_j) < 0 \\
   \end{cases}
\end{equation}

\noindent Effectively, this removes the effect of the coordinate limit geometry once motion is orthogonal or away from the limit. Finally, to use this geometry within the existing two-dimensional system, all coordinate-limiting geometries must be pulled back to root coordinates, $\q$. The weighted geometries at the root take the form,

\begin{align}
\J_i^T M_i \J_i \qdd = -\J_i^T M_i \dot{\J_i} \qd - \J_i^T M_i h_{2,i}.
\end{align}

To illustrate the effect of adding this geometry to the system, limits are placed at $\pm 4$ for both $q_1$ and $q_2$, a two-dimensional Euclidean space. Furthermor, $\lambda = 0.25$. Now the total system weighted geometries and system energy is described by

\begin{align}
M_e \qdd = \left( \M_b + \sum_i \J_i^T M_i \J_i \right) \qdd = \M_b \h_{2,b} + \sum_i -\J_i^T M_i \dot{\J_i} \qd - \J_i^T M_i h_{2,i}
\end{align}

\noindent and

\begin{align}
\Lag_{e} =  \Lag_{e,b} + \sum_i \Lag_{e,i}
\end{align}

The time evolution of the system geometry can be see in the second row of Figure \ref{fig:geometries_optimization}a. If particles are far away from their limits, then straight-line motion is observed due to the underlying Euclidean geometry. As particles approach limit boundaries, repulsion impedes motion towards the limits, and motion is ultimately redirected to limit corners before any one limit is reached. Once energized, optimization was conducted with this geometry as can be seen in Figure \ref{fig:geometries_optimization}b. Importantly, behavior is not significantly influenced when motion is away from limits or distance to limits are larger. Subtle behavioral changes can be seen by the slight bulging of the particle plume and the approach to the target is slightly more aligned with the vertical direction.

\subsubsection{Obstacle Avoidance and Goal-Reaching Optimization} \label{sec:ObstacleOptimizationExperiment}
Another canonical behavior is that of obstacle avoidance. This behavior is akin to limit avoidance except the task map is different to accommodate an object instead of a coordinate limit. For instance, the task map for a circular object can be defined as 

\begin{equation}
x = \phi(\q) = \frac{\|\q - \q_o\|}{r} - 1
\end{equation}

\noindent where $\q_o$ is the origin of the circle and $r$ is its radius. A geometry for obstacle avoidance can be constructed with the exact same form as the limit avoidance geometry,

\begin{align}
h_{2,o}(x, \dot{x}) = \lambda \dot{x}^2 \partial_x \psi(x)
\end{align}

\noindent where $\psi(x)$ is the same potential function from limit avoidance. To define priority for this geometry, the same energy structure, switching function, and metric from limit avoidance is leveraged again as well,

\begin{align}
\Lag_{e,o} = \frac{1}{2} s(\hat{\dot{x}}) \: G(x) \dot{x}^2.
\end{align}

\noindent Now, the total weighted system geometry and energy are

\begin{align}
M_e \qdd
&= \left( \M_b + \sum_i \J_i^T M_i \J_i + \J_o^T M_o \J_o \right) \qdd \\
&= \M_b \h_{2,b} + \left( \sum_i -\J_i^T M_i \dot{\J_i} \qd - \J_i^T M_i h_{2,i} \right) -\J_o^T M_o \dot{\J_o} \qd - \J_o^T M_o h_{2,o}
\end{align}

\noindent and

\begin{align}
\Lag_{e} =  \Lag_{e,b} + \sum_i \Lag_{e,i} + \Lag_{e,o}.
\end{align}

This combination of geometries and energies produces the behavior seen in the third row of Figure \ref{fig:geometries_optimization}a. Importantly, trajectories that are far from the object do not experience its effect and those closer to the object (magenta and yellow curves) are redirected. The particle belonging to the red curve was directed straight to the object's center and therefore it increasingly slows as it approaches the boundary of the object. The red particle path does subtly redirect around the object as a product of the faint interactions of the limit avoiding geometries.

Again, optimization is conducted over this increasingly more complex geometric fabric and results are shown in the third row of Figure \ref{fig:geometries_optimization}b. As seen, paths naturally bend around the obstacle and convergence to the target occurs. Importantly, despite the red and cyan curves approaching orthogonally (or nearly so) to the object, they successfully circumnavigate it and optimization succeeds. It is also possible to incorporate different task spaces and geometries to encourage early object avoidance even if the approach direction is orthogonal to the object.

\subsubsection{Randomized Vortices and Goal-Reaching Optimization} \label{sec:VortexOptimizationExperiment}
Randomized vortex geometries are added to the system as an abstraction of invoking significant and desired behavioral changes to a system. Furthermore, the vortices create a rather chaotic fabric for optimization, yet, optimization still succeeds as guaranteed by the theory. A vortex geometry is created in root coordinates as

\begin{align}
\h_{2,j}(\q, \dot{\q}) = f_j \| \qd \| ^2 \mathbf{R}_j \hat{\qd}
\end{align}

\noindent where $\mathbf{R}_j \in \mathbb{R}^{2\times 2}$ is a rotation matrix randomly selected as either

\begin{align}
\mathbf{R}_j = \pm
\begin{bmatrix}
0 & -1\\
1 & 0
\end{bmatrix}
\end{align}

\noindent and $f \in \mathbb{R}^+$ is a force magnitude uniformly drawn from the interval $f \sim \mathcal{U}(2, 10)$. A radial priority for this geometry was created from the energy tensor of the following energy,

\begin{align}
L_{e,j} = m s(\q) \qd^T \qd
\end{align}

\noindent where $m \in \mathbb{R}^+$ is a mass constant and the switching function $s(\q) \in [0,1]$ is defined as

\begin{equation}
    s(\q) =
    \left\{ \begin{array}{ll}
    \frac{1}{\|\q - \q_{o,j}\|^2} (\|\q - \q_{o,j}\| - r)^2 & \|\q - \q_{o,j}\| < r \\
    0 &\text{otherwise}
    \end{array} \right.
\end{equation}

\noindent where $\q_{o,j}$ is the effect center of a vortex and $r \in \mathbb{R}^+$ is its radius. In these experiments, $r=1$ for all vortices.

The total system weighted geometries and energy is now

\begin{align}
M_e \qdd
&= \left( \M_b + \sum_i \J_i^T M_i \J_i + \J_o^T M_o \J_o + \sum_j \M_j \h_{2,j} \right) \qdd \\
&= \M_b \h_{2,b} + \sum_i -\J_i^T M_i \dot{\J_i} \qd - \J_i^T M_i h_{2,i} -\J_o^T M_o \dot{\J_o} \qd - \J_o^T M_o h_{2,o} + \\
& \sum_j \M_j \h_{2,j}
\end{align}

\noindent and

\begin{align}
\Lag_{e} =  \Lag_{e,b} + \sum_i \Lag_{e,i} + \Lag_{e,o} + \sum_j \Lag_{e,j}.
\end{align}

As seen in the fourth row Figure \ref{fig:geometries_optimization}a, the addition of vortex geometries create turbulence to the object motion, redirecting particles in different directions as they enter a vortex zone. Despite this volatility, optimization progresses unhindered with all particles making the target location as shown in Figure \ref{fig:geometries_optimization}b. Moreover, the paths taken during optimization are also significantly different from the previous optimization results, showcasing the innate ability of geometries to influence system behavior.

\subsubsection{Attractor and Goal-Reaching Optimization} \label{sec:AttractorOptimizationExperiment}
Finally, geometries can also be constructed to facilitate optimization. The previous addition of vortex geometries produced exploratory behavior beyond more straightforward motion to the goal location. However, straight motion to the goal could be desired system behavior. This trait can easily be obtained with the addition of another geometry. Using the same task map of the optimizing potential (Equation \ref{eq:attractor_task_map}) and the potential from Equation \ref{eq:attractor_potential}, an attracting geometry can be constructed as 

\begin{align}
\label{eq:att_geometry}
\h_{2,a} = \lambda_a \| \qd \| ^2 \partial_\q \psi(\phi(\q))
\end{align}

\noindent where $\lambda_a \in \mathbb{R}^+$ is a control gain. This geometry is prioritized with the energy tensor from the energy,

\begin{align}
L_{e,a} = \xd^T \mathbf{G}_a \xd
\end{align}

\noindent where the metric $\mathbf{G}_a$ is designed as

\begin{align}
\mathbf{G}_a = s(\x) (\widebar{m} - \underline{m}) I + \underline{m} I
\end{align}

\noindent where $\widebar{m}$ and $\underline{m}$ are the upper and lower isotropic mass values and the switching function is designed as

\begin{align}
s(\x) = \frac{1}{2} (\tanh(-\alpha_s(\|\x\| - r)) + 1)
\end{align}

\noindent where $\alpha_s \in \mathbb{R}^+$ controls the rate of the switch and $r$ is the effect radius. Overall, this metric allows the priority to transition from a small to a large isotropic priority as the position nears the target. For these experiments, $\widebar{m} = 1$, $\underline{m} = 0$, $\alpha_s = 25$, $r = 5$, $\lambda_a = 7$. The parameter values for the attracting potential from Equation \ref{eq:attractor_potential} are $k=1$ and $\alpha_\psi=1$. Effectively, this geometry does not engage until the Euclidean distance between the current position and the target is within 5 units.

Altogether, the system now possesses the total weighted geometry,

\begin{align}
M_e \qdd
&= \left( \M_b + \sum_i \J_i^T M_i \J_i + \J_o^T M_o \J_o + \sum_j \M_j \h_{2,j} + \M_a \right) \qdd \\
&= \M_b \h_{2,b} + \sum_i -\J_i^T M_i \dot{\J_i} \qd - \J_i^T M_i h_{2,i} -\J_o^T M_o \dot{\J_o} \qd - \J_o^T M_o h_{2,o} + \\
& \sum_j \M_j \h_{2,j} + \M_a \h_{2,a}
\end{align}

\noindent and system energy,

\begin{align}
\Lag_e = \Lag_{e,b} + \sum_i \Lag_{e_i} + \Lag_{e,o} + \sum_j \Lag_{e,j} +  \Lag_{e,a}.
\end{align}

\noindent where the individual energies can be pulled back into root coordinates by direct substitution of the task maps and their time derivatives. The total energy Lagrangian can then be passed through the Euler-Lagrange equation to obtain system mass, $\M_e$, and force, $\f_e$. The system geometry behavior can be seen in the last row of Figure \ref{fig:geometries_optimization}a. As seen, the underlying geometry biases motion straight towards the target when within the radial threshold. The significant effect of this geometry can also be seen during optimization (last row Figure \ref{fig:geometries_optimization}b) where motion once dominated by the underlying randomized vortices are now compelled straight towards the object when particle positions are within the radial threshold.

\subsection{Planar Robot Goal-Reaching Optimization} \label{sec:PointToPointReachingExperiment}

\begin{figure}
  \centering
  \includegraphics[width=0.75\linewidth]{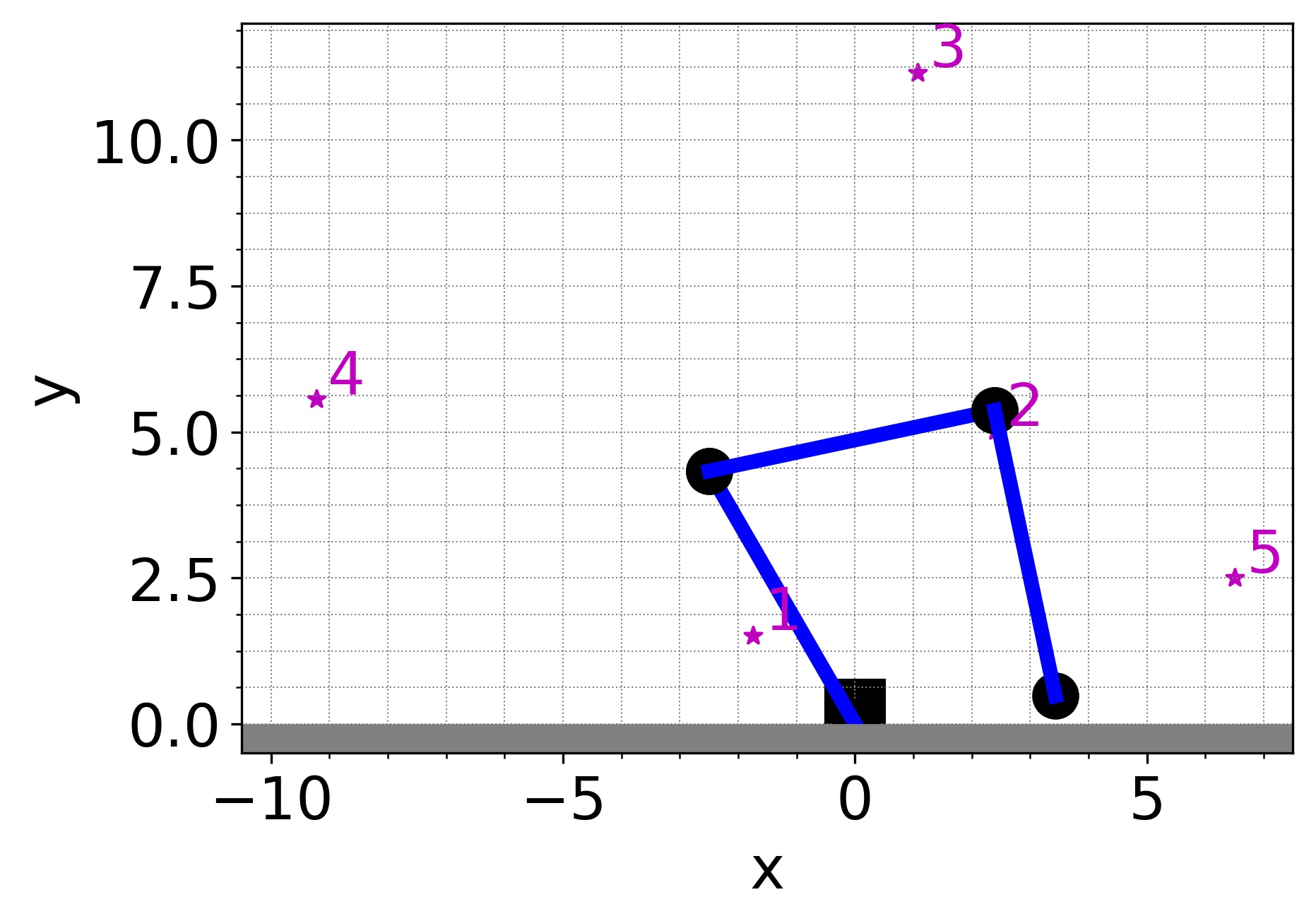}
  \caption{Planar robot goal-reaching experiment where the robot continuously moves to 5 random goals, starting from goal 1 to goal 5. This experiment is described by Section \ref{sec:PointToPointReachingExperiment}.}
  \label{fig:5goals_random_init}
\end{figure}

Here we show some experimental results on a 3-dof planar robot, as shown in Figure~\ref{fig:5goals_random_init}. The first task is simply to move the end-effector to a series of randomly picked 2-D point targets continuously. To solve this goal-reaching problem, we optimize for an attractor potential over a geometric fabric, where limit avoidance geometries, as described in Equation~\ref{eq:limit_geometry}, are used to keep each joint staying within its limit, and an attractor geometry as described in Equation~\ref{eq:att_geometry} is used to help reaching the goal. 

\newgeometry{top=0.cm, bottom=0.cm}
\begin{figure}
  \centering
  \includegraphics[height=8in]{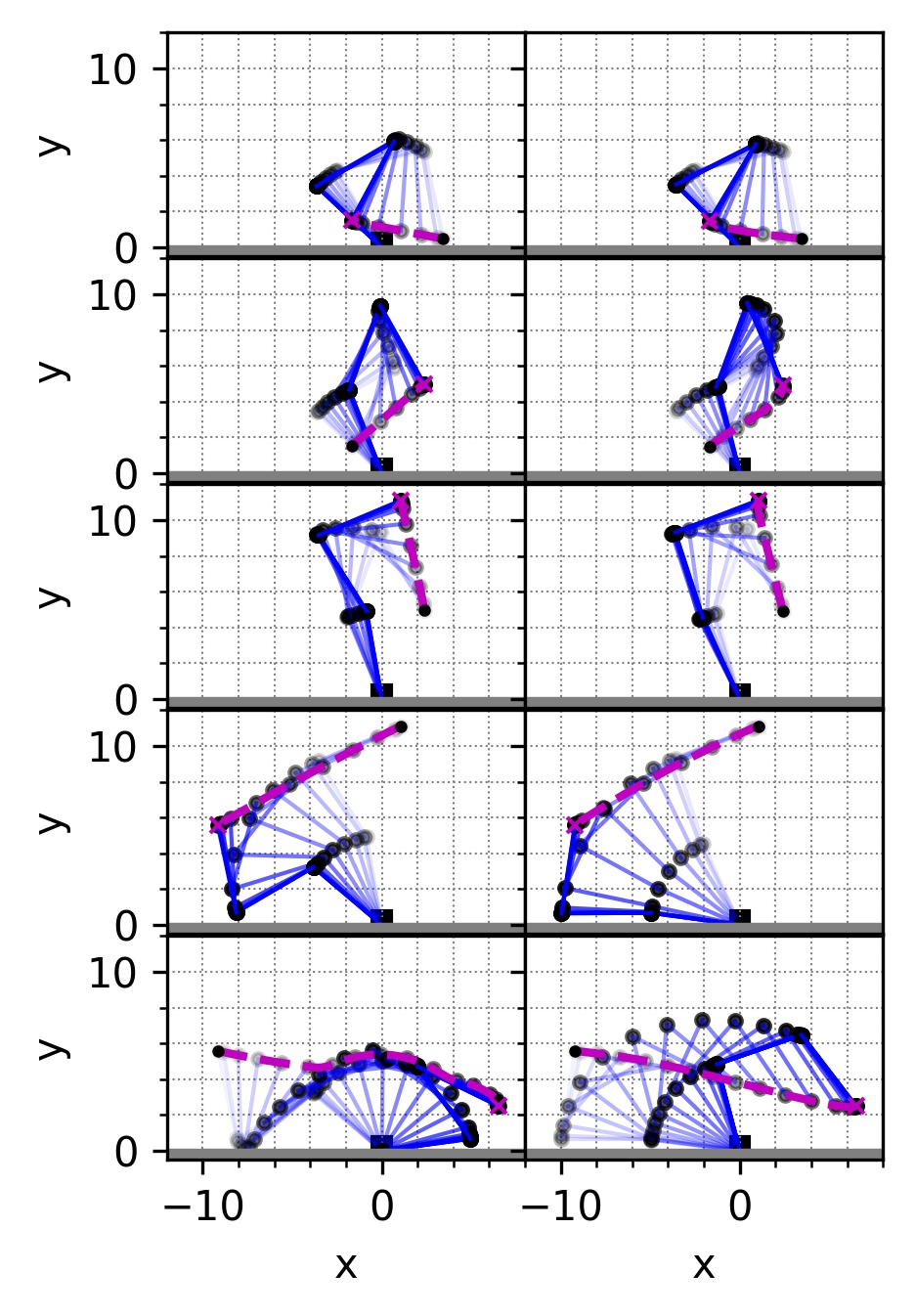}
  \caption{Continuous planar goal reaching experiment. The left and right columns shows the behavior with and without the redundancy resolution geometries, respectively. The arm's alpha transparency ranges from light at the beginning of the trajectory to dark at the end. Note in particular the last two rows: without redundancy resolution the arm bends over and ultimately collapses on itself, while with redundancy resolution it successfully retains its overall ``ready'' configuration. See Section \ref{sec:PointToPointReachingExperiment} for details.}
  \label{fig:5goals_compare}
\end{figure}
\restoregeometry

The main difference between this planar arm setup and the above 2D point mass fabric is that the arm has redundancy. That means both that we need to add the attractor to the end-effector space and that we need to resolve the redundancy with an extra geometric fabric component. To resolve the redundancy issue, we add a redundancy resolution geometry to the fabric coaxing the system toward a default arm configuration. This geometry can be expressed as follows:
\begin{align}
\label{eq:dc_geometry}
\h_{\q,\qd} = \| \qd \| ^2 (\q_0 - \q)
\end{align}
where $\q_0$ is a nominal default configuration such as the one shown in Figure~\ref{fig:5goals_random_init}, and it usually represents the robot ``at ready". This geometry is prioritized with the energy tensor from the energy,

\begin{align}
L_{e, dc} = \qd^T \mathbf{G}_{dc} \qd
\end{align}


\noindent where the metric $\mathbf{G}_{dc}$ is simply an identity matrix scaled by a constant $\lambda_{dc}$ as follows:
\begin{align}
\mathbf{G}_{dc} = \lambda_{dc} I
\end{align}
To demonstrate the effect of the redundancy resolution geometry, we run two experiments, in which the robot continuously moves its end-effector to 5 randomly picked targets, starting from 1 to 5, indicated by asterisks as shown in Figure~\ref{fig:5goals_random_init}. In the first experiment, goal-reaching optimization is performed over the fabric containing only limit avoidance geometries and an attractor geometry (no redundancy resolution). In the second experiment, we add a redundancy resolution geometry to the fabric. When moving continuously between random targets we would expect without redundancy resolution the arm will fold over into awkward configurations, while with it, the arm will run forever resolving itself back to the general default posture whenever possible. The experiment results are shown in Figure~\ref{fig:5goals_compare}. Figures from the top to the bottom show the robot moving sequentially from goals 1 to 5 (row 1 moves from an initial configuration to goal 1, row 2 moves from that resulting configuration to goal 2, etc.). The left column shows results for the experiment without redundancy resolution, and the right column shows results with redundancy resolution. In each figure, alpha transparency of the arm indicates progress along the trajectory, starting from light at the beginning and ranging to dark at the end. 

As expected, we observe that with redundancy resolution the robot consistently resolves back to the default ``ready'' configuration. The behavior is most apparent in the final two rows where without redundancy resolution (left) the arm folds over on itself when reaching from goal 4 to goal 5. With redundancy resolution (right) the arm successfully resolves back to ``ready'' en route.



\subsection{Planar Robot Behavior Shaping Optimization} \label{sec:EndEffectorBehaviorExperiment}

In this section, we show some experimental results on shaping robotic behavior with geometric fabrics. Instead of reaching the goal along the floor in a straight line motion at the end-effector task space, we want to shape the behavior of the system so that it can en route toward its target. In this case, we want the end-effector to first lift off the floor, then move to the target region horizontally parallel to the floor, and then descend towards the target from above once close enough. In a real life application, this behavior is designed for the robot to reach the goal without hitting any obstacles on the floor. We can construct this behavior using two geometries, a floor lifting geometry and a goal attraction geometry. As the name implies, the floor lifting geometry lifts the end-effector off the floor, and can be formulated as follows:
\begin{align}
\label{eq:fl_geometry}
\h_{\x,\xd} = \| \xd \| ^2 \hat{n}
\end{align}
where $\x$ and $\xd$ are expressed in the end-effector taskspace, and $\hat{n}$ is the normal vector to the floor pointing up. This geometry is prioritized with the energy tensor from the energy,

\begin{figure}
  \centering
  \includegraphics[width=0.75\linewidth]{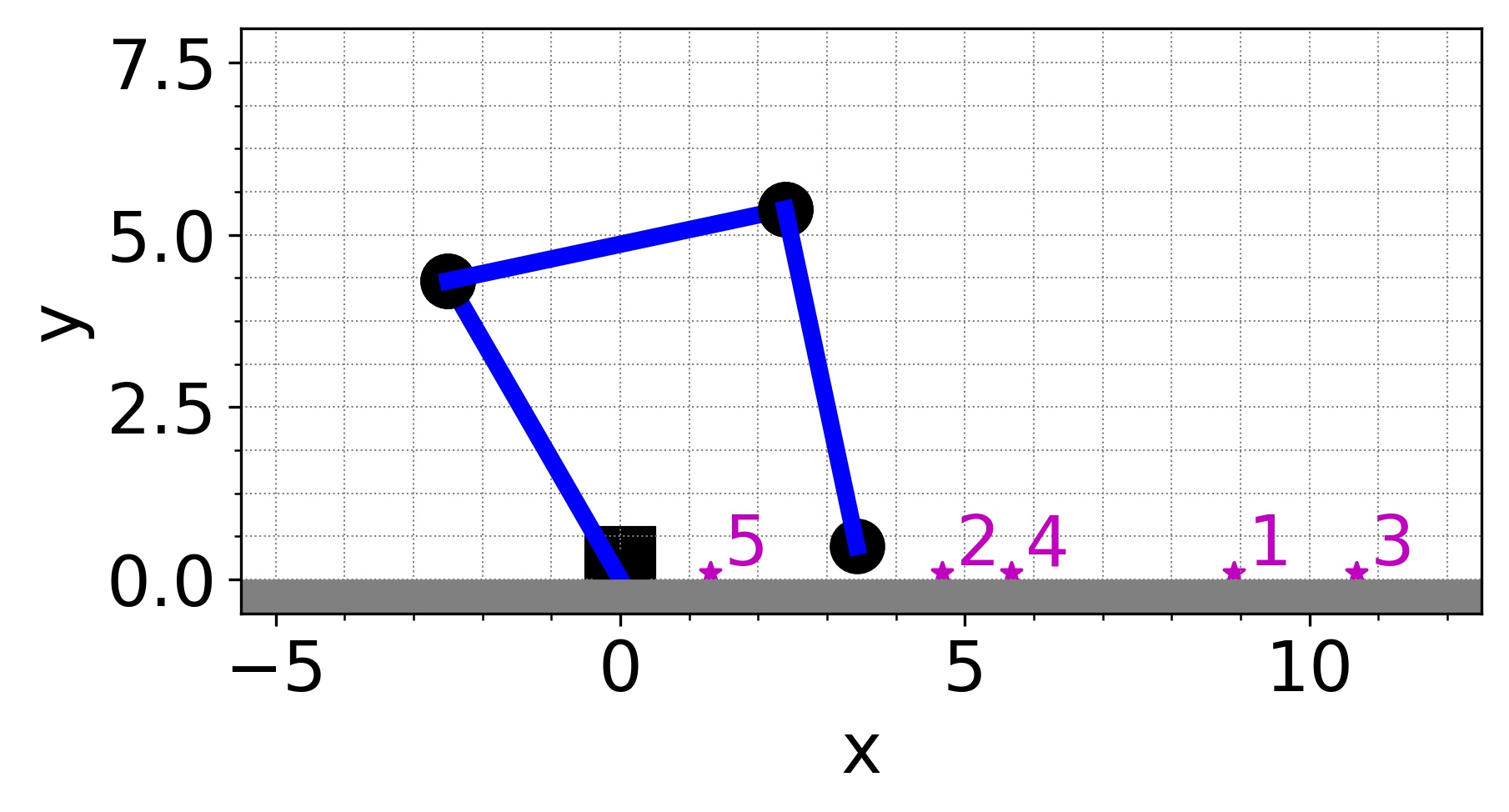}
  \caption{Planar robot behavior shaping experiment where the robot continuously moves to 5 random goals on the ground with shaped behavior, starting from 1 to 5 indicated by asterisks. This experiment is described by Section \ref{sec:EndEffectorBehaviorExperiment}.}
  \label{fig:5goals_init}
\end{figure}

\begin{align}
L_{e, fl} = \xd^T \mathbf{G}_{fl} \xd
\end{align}

\noindent where the metric $\mathbf{G}_{fl}$ is designed as follows
\begin{align}
\mathbf{G}_{fl} = \lambda_{fl} \exp{(-\frac{h}{\sigma_{fl}})} I
\end{align}
in which $\lambda_{fl}$ is constant scalar greater than zeros, and $\sigma_{fl}$ is a length scale. The goal attractor geometry is similar to the one used in the previous section, but with a slight change in the formulation,
\begin{align}
\label{eq:bs_att_geometry}
\h_{\x,\xd} = \| \xd \| ^2 (\x_g - \x)
\end{align}
where $\x$ and $\xd$ are expressed in the end-effector taskspace, and $\x_g$ is the goal position. 

\newgeometry{top=0.cm, bottom=0.cm}
\begin{figure}
  \centering
  \includegraphics[height=8in]{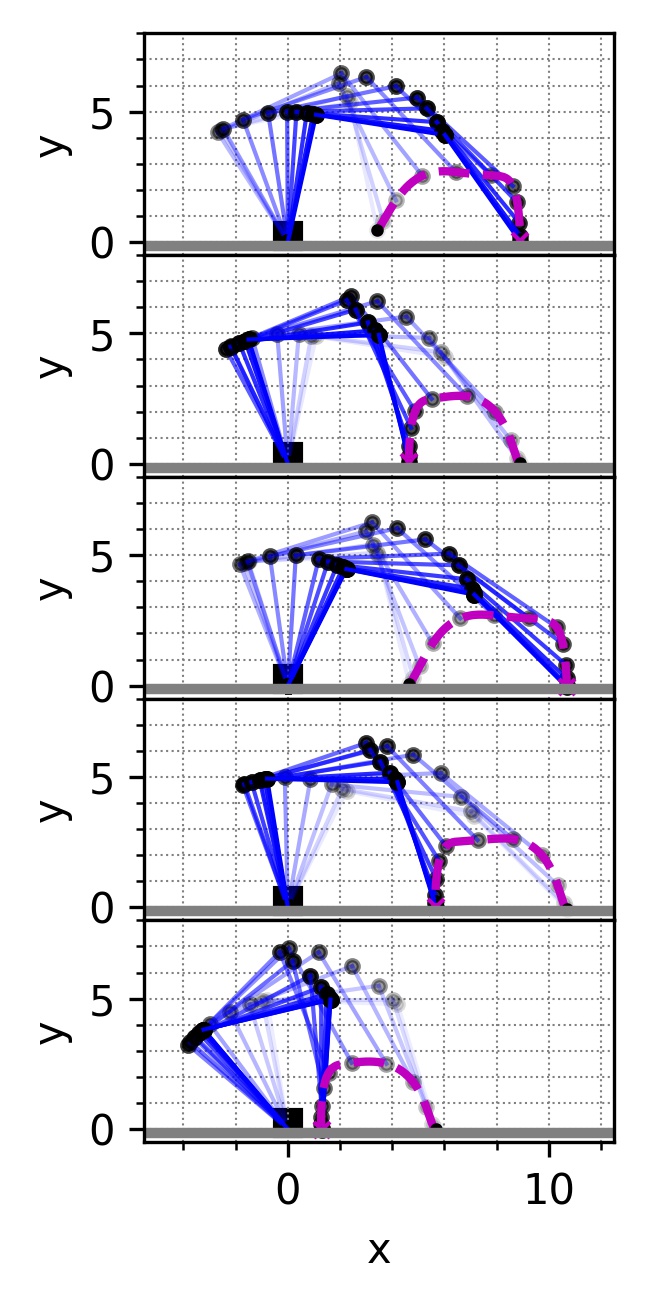}
  \caption{Planar robot behavior shaping experiment with geometric fabrics. From the top to the bottom, figures show the robot moving sequentially from goal 1 to goal 5, each time starting from the final configuration of the last episode. The arm's alpha transparency ranges from light at the beginning of the trajectory to dark at the end. As designed, geometric fabrics on the end-effector shape the system to first lift the end-effector from the floor, then proceed horizontally a fixed distance above the floor before lowering to the target from above. See Section \ref{sec:EndEffectorBehaviorExperiment} for details.}
  \label{fig:5goals_behavior_shaping}
\end{figure}
\restoregeometry

\noindent This geometry is prioritized with the energy tensor from the energy,
\begin{align}
L_{e, att} = \xd^T \mathbf{G}_{att} \xd
\end{align}
\noindent where the metric $\mathbf{G}_{att}$ is designed as follows
\begin{align}
\mathbf{G}_{att} = \lambda_{att} \exp{(-\frac{s^2}{2\sigma_{att}^2})} I
\end{align}
in which $\lambda_{att}$ is constant scale greater than zeros, $\sigma_{att}$ is a length scale, and $s$ is the horizontal distance between the end-effector and the goal.

To accomplish a goal reaching task with shaped behavior, we optimize an attractor potential over a geometric fabric which includes joint limit avoidance geometries, a default configuration geometry, a floor lifting geometry and a goal attractor geometry. We run an experiment, in which the robot continuously moves its end-effector to 5 randomly picked targets on the floor, masked as 1 to 5 indicated by asterisks as shown in Figure~\ref{fig:5goals_init}. The experimental results are shown in Figure~\ref{fig:5goals_behavior_shaping}, in which figures from the top to the bottom show that the robot moves from goal 1 to goal 5. As expected, the robot reaches all the goals with the desired behavior.


\section{Conclusions}

This work has been an illuminating exploration into the experimentally verifiable theoretical properties of structured second-order systems, and from our deepened understanding of the problem we gain a collection of concrete and powerful tools for behavior design.

Our experimental results confirm predicted theoretical properties such as geometric consistency, energization commutes with pullback of full rank systems, the ability to optimize across arbitrary (even random) fabrics. The provided framework also enables the sequential construction of behavioral systems, where designers can add more constraints, layer-by-layer, without having to re-tune earlier components. 

Our next steps are to study constrained optimization in greater generality (beyond boundary constraints) and to explore how the structure studied in this paper affects policy learning paradigms. Regarding the latter, optimization fabrics constitute nominal behaviors shared across tasks that encode commonalities among problems---the fabric itself create the underlying behavioral dynamics and task potentials (objective functions) act to push the system away from those nominal behaviors to (provably) solve specific tasks. This structure may, therefore, play the role of a well-informed policy class prior that can be trained over time across multiple related problems.

\pagebreak[4]
\appendix
\appendixpage

\section{Concrete derivations on manifolds}
\label{apx:Manifolds}

Nonlinear geometry is most commonly constructed in terms of smooth manifolds, often in an abstract, coordinate-free, form. To make the topic more accessible, we will stick to coordinate descriptions and standard vector notations from advanced calculus. Analogous to how standard classical equations of motion are expressed in generalized coordinates convenient to the problem (Cartesian coordinates, polar coordinates, robotic joint angles, etc.) and understood to represent concrete physical phenomena independent of those coordinates, we take the same model here. Our constructions of nonlinear geometry will be made exclusively in terms of coordinates to keep the expressions and notation familiar, and practitioners are free to change coordinates as needed as the system moves across the manifold.  

Formally, the equations we describe in this paper are \textit{covariant}, which means they maintain their form under changes of coordinates. See \cite{ratliff2020SpectralSemiSprays} for an in depth discussion of covariance and the use of transform trees to guarantee covariant transformation. Instead of ensuring explicitly all objects used in equations are coordinate free, we define quantities in terms of clearly coordinate free quantities such as lengths. For instance, we define geometry in terms of minimum length criteria, so as long as the length measure transforms properly so it remains consistent under changes of coordinates, the geometric equations should be independent of coordinates.

Manifolds will be defined in the traditional way (see \cite{LeeSmoothManifolds2012} for a good introduction), but for our purposes, we will consider them $d$-dimensional spaces $\mathcal{X}$ with elements identified with $\x\in\R^d$ in $d$ coordinates. Often we implicitly assume a system evolves over time $t$ in a trajectory $\x(t)$ with velocity $\xd = \frac{d\x}{dt}$, and say that coordinate velocity vector $\xd$ is an element of the tangent space $\mathcal{T}_\x\mathcal{X}$. When discussing general tangent space vectors, we often us a separate notation $\vv\in\mathcal{T}_\x\mathcal{X}$ to distinguish it from being a velocity of a specific trajectory.

Manifolds with a boundary are common modeling tools in this work as well. For instance, both joint limits and obstacles form boundaries in a manifold. If $\mathcal{X}$ denotes a manifold, its boundary is denoted $\partial\mathcal{X}$ and is assumed to form a smooth lower-dimensional submanifold with cooresponding tangent space $\mathcal{T}_\x\partial\mathcal{X}$. 

\section{Advanced calculus notation}
\label{apx:CalculusNotation}

Higher-order tensors are common in differential geometry, so frequently when dealing with coordinates, explicit indices are exposed and summed over to handle the range of possible combinations. To avoid clutter, the Einstein convention is then used to unambiguously drop summation symbols. 

However, the resulting index notation is unfamiliar and takes some getting used to. Alternatively, the matrix-vector notation found in many advanced calculus and engineering texts is often much simpler and concise. The key to this notation's simplicity is the associativity of matrix products. When at most two indices are involved, by arranging the components of our indexed objects into matrices, with one-indexed vectors being column vectors by default, we can leverage this associativity of matrix products to remove the indices entirely while ensuring expression remain unambiguous with regard to order of operation. We, therefore, use the simpler and more compact matrix-vector notation wherever possible, with a slight extension for how to deal with added indices beyond 2 that might arise from additional partials as discussed below. The majority of our algebra can be expressed using just two indices allowing us to remain within the matrix-vector paradigm.
 
Whenever we take a partial derivative, a new index is generated ranging over the individual dimensions of the partial derivative. For instance, a partial derivative of a function $\partial_\x f(\x)$ where $f:\R^n\rightarrow\R$ is a list of $n$ partial derivatives, one for each dimension of $\x$. Likewise, if $\g:\R^n\rightarrow\R^m$, there are again $n$ partials in $\partial_\x\g$, but this time each of them has the same dimensionality as the original $\g$. $\g$ has its own index ranging over its $m$ codomain dimensions, and now there's a new second index ranging over the $n$ partials. 

We use the convention that if the partial derivative generates a first index (as in $\partial_\x f$), it is oriented as a column vector by default (in this case $n$ dimensional), with its transpose being a row vector. If it generates a second index (as in $\partial_\x \g$), the second index creates a matrix (in this case $m\times n$ dimensional) so that the original vector valued function's orientation is maintained. Specifically, if originally the vector-valued function was column oriented, then each partial will be column oriented and they will be lined up in the matrix so that the new index ranges over the columns. And if the original vector-valued function is row oriented, then the partials will be row vectors and stacked to make rows of a matrix, so the new index will range over the first index of the resulting matrix.

We use the compact notation $\partial_\x$ rather than $\frac{\partial}{\partial \x}$ so multiple partial derivatives unambiguously lists the partials in the order they're generating indices. The notation $\partial_{\x\y} h(\x,\y)$ where $h:\R^m\times\R^n\rightarrow\R$, therefore, means we first generate an index over partials of $\x$ and then generate an index over partials of $\y$. That means $\partial_{\x\y}h = \partial_\y\big(\partial_\x h\big)$ is an $m\times n$ matrix since $\partial_\x h$ is first generated as an $n$-dimensional column vector. When the partials are over two of the same variable denoting a Hessian $\mH_f$, we use the notation $\mH_f = \partial^2_{\x\x} f$ using a squared exponent to emphasize that the partials are not mixed.

For partials generating up to only the first two indices, matrix algebra governs how they interact with surrounding matrices and column or row vectors, with the partial derivative operator taking priority over matrix multiplication (we use parentheses when the product should be included in the partial). For instance, $\partial_\x \g\:\xd = \J_\g\xd$ is the Jacobian of $\g$, with partials ranging over the columns (second index) and each partial constituting a full column, multiplied by the column vector $\xd$. Likewise, $\frac{1}{2}\xd^\tr\partial^2_{\xd\xd}\Lag_e\xd$ is a squared norm of $\xd$ with respect to the symmetric matrix $\partial^2_{\xd\xd}\Lag_e = \partial_\xd\partial_\xd\Lag_e$ of second partials. The first $\partial_\xd\Lag_e$ creates a column vector (it's a first index), and the second $\partial_\xd\big(\partial_\xd\Lag_e\big)$ creates a matrix, with columns constituting the partials of the vector of first partials.

We use bold font to emphasize that these objects with one or two indices are treated as matrices, with one index objects assumed to be column oriented by default. Transposition operates as standard in matrix algebra, swapping the indices in general, with column vectors becoming row vectors and vice versa.

Beyond two indices, we use the convention that a partial derivative simply generates a new index ranging over the partials. By default, multiplication on the right by another indexed object (e.g. vector or matrix) will contract (sum products of matching index values) across this new index and the first index in the right operand. For instance, if $\M:\R^d\rightarrow\R^{m\times n}$, and $\qd\in\R^d$ is a velocity, then $\partial_\q\M\:\qd = \sum_{i=1}^d\frac{\partial\M}{\partial q^i} \dot{q}^i$, unambiguously. Note that beyond two indices, associativity no longer holds in general, so parentheses must be used to disambiguate wherever necessary.

\vspace{10pt}
\noindent Examples:
\begin{enumerate}
\item $\partial_\x f(\x,\y)$ is a column vector
\item $\partial_{\x\y} f(\x,\y)$ is a matrix with first index ranging over partials of $\x$ and second index ranging over partials of $\y$.
\item When $\f:\R^n\rightarrow\R^m$ the matrix $\partial_\x \f$ is an $m\times n$ Jacobian matrix, and $\partial_\x \f^\tr$ is an $n\times m$ matrix, the transpose of the Jacobian matrix.
\item $\frac{d}{dt} \f(\x) = \partial_\x \f\:\xd$ and $\frac{d}{dt} \f^\tr(\x) = \xd^\tr \partial_\x \f^T = \xd^\tr \big(\partial_\x \f\big)^\tr$.
\item $\partial_\x g\big(\f(\x)\big) = \big(\partial_\x \f\big)^\tr \partial_\y g|_{\f(\x)} = J_\f^\tr \partial_\y g(\f(\x))$ with $\y = \f(\x)$.
\end{enumerate}
Often a gradient is denoted $\nabla_\x f(\x)$, but with these conventions outlined above, we use simply $\partial_\x f(\x)$ to avoid redundant notation. We additionaly frequently name common expressions for clarity, such as 
\begin{enumerate}
    \item $\p_e = \partial_\xd\Lag_e$
    \item $\M_e = \partial^2_{\xd\xd}\Lag_e = \partial_\xd \p_e$
    \item $\J_{\p_e} = \partial_\x \p_e$ (even though $\p_e$ is a function of both position and velocity, we use $\J$ to denote the position Jacobian)
    \item $\g_e = \partial_\x\Lag$
\end{enumerate}
These vector-matrix definitions $\M_e,\J_{\p_e},\g_e$ make it more clear what the general size and orientation of the objects $\partial^2_{\xd\xd}\Lag, \partial_{\xd\x}\Lag = \partial_\x \p_e$, and $\partial_\x\Lag$ are.

The equations of motion in different forms would be:
\begin{align}
    &\partial^2_{\xd\xd}\mathcal{L}\;\xdd +  \partial_{\xd\x}\mathcal{L}\:\xd - \partial_\x \mathcal{L} = \zero.\\
    &\M_e\xdd + \J_{\p_e}\xd - \g_e = \zero.
\end{align}
Using the named quantities, that structure of the final expression is clear at a glance. Similarly, we have
\begin{align}
E(\x,\xd) &= \frac{1}{2}\xd^\tr\:\partial^2_{\xd\xd}\mathcal{L}\: \xd = \frac{1}{2}\xd^\tr\M_e\xd,
\end{align}

Example of algebraic operations using this notation (taken from a common calculation involving the time invariance of Finsler energy functions):
\begin{align}
    &\xd^\tr \partial_{\xd\x}\Lag\:\xd - 2\partial_\x\Lag_e^\tr\xd
    = \big( \xd^\tr \partial_\x\partial_\xd\Lag_e - 2\partial_\x \Lag_e^\tr \big)\xd\\
    &\ \ = \big( \partial_\x(\xd^\tr \partial_\xd\Lag_e)^\tr - 2\partial_\x \Lag_e^\tr \big)\xd\\
    &\ \ = \partial_\x\big(\xd^\tr\partial_\xd \Lag_e - 2\Lag_e\big)^\tr \xd \\
    &\ \ = \partial_\x\big(2\Lag_e - 2\Lag_e\big)^\tr\xd\\
    &\ \ = \zero,
\end{align}
where we use $\Lag_e = \frac{1}{2}\xd^\tr\partial_\xd\Lag_e$ in that last line.

\bibliographystyle{plain}
\bibliography{refs}

\end{document}